\documentclass{article}

\usepackage{microtype}
\usepackage{graphicx}
\usepackage{subfigure}
\usepackage{booktabs} 

\usepackage{amsmath,amsfonts,bm}

\newcommand{\DS}{$\mathcal{D}_S$}
\newcommand{\DT}{$\mathcal{D}_T$}

\def\figref#1{figure~\ref{#1}}

\def\secref#1{section~\ref{#1}}

\def\eqref#1{equation~\ref{#1}}

\def\1{\bm{1}}

\DeclareMathAlphabet{\mathsfit}{\encodingdefault}{\sfdefault}{m}{sl}
\SetMathAlphabet{\mathsfit}{bold}{\encodingdefault}{\sfdefault}{bx}{n}

\newcommand{\R}{\mathbb{R}}

\usepackage{amsthm, amsmath}
\usepackage{amssymb}
\usepackage{natbib}
\usepackage{graphicx}
 \usepackage{multirow}
 \usepackage{wrapfig}

\newtheorem{theorem}{Theorem}[section]

\newtheorem{lemma}[theorem]{Lemma}

\newtheorem{definition}[theorem]{Definition}

\newcommand{\reals}{\mathbb{R}}

\newcommand\revision[1]{\textcolor{black}{#1}}

\newcommand{\be}{\mathbf{e}}
\newcommand{\bx}{\mathbf{x}}
\newcommand{\bw}{\mathbf{w}}

\newcommand{\bb}{\mathbf{b}}
\newcommand{\bu}{\mathbf{u}}

\newcommand{\bz}{\mathbf{z}}

\newcommand{\bh}{\mathbf{h}}
\newcommand{\by}{\mathbf{y}}

\renewcommand{\secref}[1]{Sec.~\ref{#1}}

\renewcommand{\figref}[1]{Fig.~\ref{#1}}
\renewcommand{\eqref}[1]{Eq.~(\ref{#1})}
\newcommand{\lemref}[1]{Lemma~\ref{#1}}

\newcommand{\thmref}[1]{Thm.~\ref{#1}}

\newcommand{\appref}[1]{Appendix~\ref{#1}}

\usepackage{hyperref}
\usepackage{url}
\usepackage{hyperref}

\usepackage[accepted]{icml2021}

\icmltitlerunning{From Local Structures to Size Generalization in Graph Neural Networks}

\begin{document}

\twocolumn[
\icmltitle{From Local Structures to Size Generalization in Graph Neural Networks}




\begin{icmlauthorlist}
\icmlauthor{Gilad Yehudai}{nvidia}
\icmlauthor{Ethan Fetaya}{biu}
\icmlauthor{Eli Meirom}{nvidia}
\icmlauthor{Gal Chechik}{nvidia,biu}
\icmlauthor{Haggai Maron}{nvidia}
\end{icmlauthorlist}

\icmlaffiliation{nvidia}{NVIDIA}
\icmlaffiliation{biu}{Bar-Ilan University}

\icmlcorrespondingauthor{Gilad Yehudai}{gilad.yehudai@weizmann.ac.il}

\icmlkeywords{Machine Learning, ICML}

\vskip 0.3in
]



\printAffiliationsAndNotice{}  

\begin{abstract}
Graph neural networks (GNNs) can process graphs of different sizes, but their ability to generalize across sizes, specifically from small to large graphs, is still not well understood. 
In this paper, we identify an important type of data where generalization from small to large graphs is challenging: graph distributions for which the local structure depends on the graph size. This effect occurs in multiple important graph learning domains, including social and biological networks. We first prove that when there is a difference between the local structures, GNNs are not guaranteed to generalize across sizes: there are "bad" global minima that do well on small graphs but fail on large graphs. We then study the size-generalization problem empirically and demonstrate that when there is a discrepancy in local structure, GNNs tend to converge to non-generalizing solutions. Finally, we suggest two approaches for improving size generalization, motivated by our findings. Notably, we propose a novel Self-Supervised Learning (SSL) task aimed at learning meaningful representations of local structures that appear in large graphs. Our SSL task improves classification accuracy on several popular datasets.

\end{abstract}

\section{Introduction}





Graphs are a flexible representation, widely used for representing diverse data and phenomena. In recent years, graph neural networks (GNNs), deep models that operate over graphs, have become the leading approach for learning on graph-structured data \citep{bruna2013spectral,kipf2016semi,velivckovic2017graph,gilmer2017neural}.

In many domains, graphs vary significantly in size. This is the case in molecular biology, where molecules are represented as graphs and their sizes vary from few-atom compounds to proteins with  thousands of nodes. 
Graph sizes are even more heterogeneous in social networks, {ranging from dozens of nodes to  billions of nodes. 
Since a key feature of GNNs is that they can operate on graphs regardless of their size,} 
a fundamental question arises: \textbf{"When do GNNs generalize to graphs of sizes that were not seen during training?"}. 



Aside from being an intriguing theoretical question, the size-generalization problem has important practical implications. In many domains, it is hard to {collect ground-truth labels for} large graphs. For instance, many combinatorial optimization problems can be represented as graph {classification} problems, but labeling large graphs for training may require solving large and hard optimization problems. In other domains, it is often very hard for human annotators to correctly label complex networks. It would therefore be highly valuable to develop techniques that can train on small graphs and generalize to larger graphs. This first requires that we develop an understanding of size generalization.

\begin{figure}[t]
    \centering
    \includegraphics[width=0.9\columnwidth]{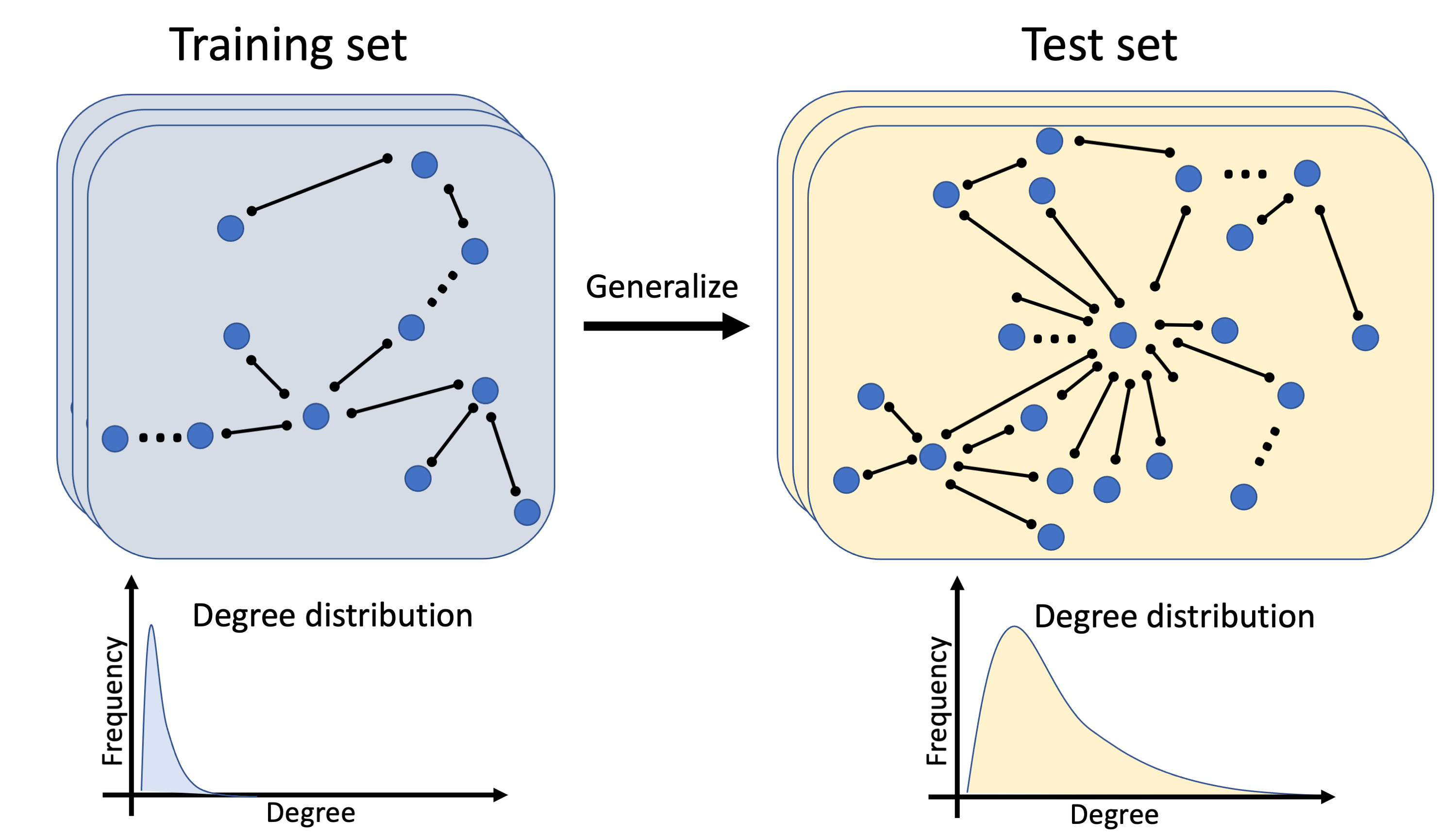}
    \caption{We study the ability of GNNs to generalize from small to large graphs, focusing on graphs in which the local structure depends on the graph size. The figure shows two graph distributions that differ in size and degree distribution. We show that when the local structures in the test set are different from the local structures in the training set, it is difficult for GNNs to generalize. Additionally, we suggest ways to improve generalization. }
     \label{fig:fig1}
\end{figure}

In some cases, GNNs can naturally generalize to graphs whose size is different from what they were trained on, but it is largely unknown when such generalization occurs. Empirically, several papers report good size-generalization performance \citep{li2018combinatorial, luz2020learning, sanchez2020learning}. Other papers \cite{velickovic2019deep, khalil2017learning, joshi2020learning} show that size generalization can be hard. Recently, \citet{xu2020neural} provided theoretical evidence of size generalization capabilities in one-layer GNNs.

The current paper characterizes an important type of graph distributions where size generalization is challenging. Specifically, we analyze graphs for which the distribution of local structures (defined formally in \secref{sec:local graph patterns}) depends on the size of the graph.
See \figref{fig:fig1} for an illustrative example. This dependency is prevalent in a variety of graphs, including for instance, the preferential attachment (PA) model \cite{barabasi1999emergence}, which captures graph structure in social networks \cite{barabasi2002evolution}, biological networks \cite{eisenberg2003preferential,light2005preferential}  and internet link data \cite{capocci2006preferential}.
In PA, the maximal  node degree grows with the graph size. As a second example, in a graph representation of dense point clouds, the node degree grows with the cloud density, and hence with the graph size  \cite{hermosilla2018monte}. 




To characterize generalization to new graph sizes,  
we first formalize a representation of local structures that we call $d$-patterns, inspired by \cite{weisfeiler1968reduction,morris2019weisfeiler, xu2018powerful}. $d$-patterns generalize the notion of node degrees to a $d$-step neighborhood of a given node, capturing the values of a node and its $d$-step neighbors, as seen by GNNs. We then prove that even a small discrepancy in the distribution of $d$-patterns between the test and train distributions may result in weight assignments that do not generalize well. Specifically, it implies that when training GNNs on small graphs there exist "bad" global minima that fail to generalize to large graphs. 

We then study empirically the relation between size generalization and $d$-pattern discrepancy in synthetic graphs where we control the graph structure and size. We find that as $d$-pattern discrepancy grows, the generalization of GNNs to new graph sizes deteriorates.
%

%
%
Finally, we discuss approaches for improving size generalization. 
We {take a self-supervised learning approach and} propose a novel pretext task aimed at learning useful $d$-pattern representations from both small and large graphs. We show that when training on labeled small graphs and with  our new self-supervised task on large graphs, classification accuracy increases on large graphs by  $4\%$ on average on real datasets. 


This paper makes the following contributions: (1) We identify a family of important graph distributions where size generalization is difficult, using a combination of theoretical and empirical results. (2) We suggest approaches for improving size generalization when training on such distributions and show that they lead to a noticeable performance gain. 
The ideas presented in this paper can be readily extended to other graph learning setups where there is a discrepancy between the local structures of the train and test sets. 

\section{Preliminaries}\label{sec:preliminaries}

\textbf{Notation.} We denote by $\{(a_1, m_{a_1}),\dots,(a_n, m_{a_n})\}$ a \textbf{multiset}, that is, a set where we allow multiple instances of the same element. Here $a_1,\dots,a_n$ are distinct elements, and $m_{a_i}$ is the number of times $a_i$ appears in the multiset. Bold-face letters represent vectors. 

\textbf{Graph neural networks.}
In our theoretical results, we focus on the message-passing architecture from \cite{morris2019weisfeiler}. Let $G=(V,E)$ be a graph, and for each node $v\in V$ let $\bh^{(0)}_v\in\mathbb{R}^{d_0}$ be a node feature vector and $\mathcal{N}(v)$ its set of neighbors. The $t$-th layer of first-order GNN is defined as follows for $t>0$:
\begin{equation*}
    \bh^{(t)}_v = \sigma\left( W_2^{(t)} \bh^{(t-1)}_v + \sum_{u\in\mathcal{N}(v)} W_1^{(t)} \bh^{(t-1)}_u + \bb^{(t)}\right).
\end{equation*}

Here, $W_1^{(t)},W_2^{(t)}\in\mathbb{R}^{d_{t}\times d_{t-1}},~ \bb^{(t)}\in \mathbb{R}^{d_t}$ denotes the parameters of the $t$-th layer of the GNN, and $\sigma$ is some non-linear activation (e.g ReLU). It was shown in \cite{morris2019weisfeiler}
that GNNs composed from these layers have maximal expressive power with respect to all message-passing neural networks. 
For node prediction, the output of a $T$-layer GNN for node $v$ is $\bh_v^{(T)}$. For graph prediction tasks an additional readout layer is used: $g^{(T)} = \sum_{v\in V} \bh_v^{(T)}$, possibly followed by a fully connected network.

\textbf{Graph distributions and local structures.} In this paper we focus on graph distributions for which the local structure of the graph (formally defined in \secref{sec:local graph patterns}) depends on the graph size. A well-known distribution family with this property is $G(n,p)$ graphs, also known as Erd\H{o}s-R\'enyi. A graph sampled from $G(n,p)$ has $n$ nodes, and edges are drawn i.i.d. with probability $p$. The mean degree of each node is $n\cdot p$; hence fixing $p$ and increasing $n$ changes the local structure of the graph, {specifically the node degrees}. 

As a second example, we consider the preferential attachment model \cite{barabasi1999emergence}. Here, $n$ nodes are drawn sequentially, and each new node is connected to exactly $m$ other nodes, where the probability to connect to other nodes is proportional to their degree. As a result, high degree nodes have a high probability that new nodes will connect to them. Increasing the graph size, causes the maximum degree in the graph to increase, and thus changes its local structure. We also show that, in real datasets, the local structures of small and large graphs differ. This is further discussed in \secref{sec:improve size gen} and \appref{appen:counting patterns}.

\section{Overview}
\subsection{The size-generalization problem}

We are given two distributions over graphs $P_{1}, P_{2}$ that contain small and large graphs accordingly, and a task that can be solved for all graph sizes using a GNN. We train a GNN on a training set $\mathcal{S}$ sampled i.i.d. from $P_{1}$ and study its performance on $P_{2}$.
In this paper, we focus on distributions that have a high discrepancy between the local structure of the graphs sampled from $P_1$ and $P_2$. 

\textbf{Size generalization is not trivial.} 
Before we proceed with our main results, we argue  that even for the simple regression task of counting the number of edges in a graph, which is solvable for all graph sizes by a 1-layer GNN, GNNs do not naturally generalize {to new sizes}. Specifically, we show that training a 1-layer GNN on a non-diverse dataset reaches a non-generalizing solution with probability 1 over the random initialization. In addition, we show that, in general, the generalizing solution is not the least L1 or L2 norm solution, hence cannot be reached using standard regularization methods. See full derivation in \appref{appen:linear GNN}.

\subsection{Summary of the main argument} 

This subsection describes the main flow of the next sections in the paper. We explore the following arguments:

\textbf{(i) $d$-patterns are a correct notion for studying the expressivity of GNNs.} To study size generalization, we introduce a concept named $d$-patterns, which captures the local structure of a node and its $d$-step neighbors, as captured by GNNs. This notion is formally defined in Section \ref{sec:local graph patterns}. For example, for graphs without node features, a $1$-pattern {of a node} represents {its degree}, and {its} $2$-pattern {represents} its degree and the set of {degrees of its immediate neighbors}. We argue that $d$-patterns are a natural abstract concept for studying the expressive power of GNNs: first, we extend a result by \cite{morris2019weisfeiler} and prove that $d$-layer GNNs  (with an additional node-level network) can be programmed to output any value on any $d$-pattern independently. Conversely, as shown in \cite{morris2019weisfeiler}, GNNs output a constant value when given nodes with the same $d$-pattern, meaning that the expressive power of GNNs is limited by their values on $d$-patterns.

\textbf{(ii) $d$-pattern discrepancy implies the existence {of} bad global minima.} In Section \ref{sec:bad_global}, we focus on the case where graphs in the test distribution contain $d$-patterns that are not present in the train distribution.
In that case, we prove that for any graph task solvable by a GNN, there is a weight assignment that succeeds on training distribution and fails on the test data. In particular, when the training data contains small graphs and the test data contains large graphs, if there is a $d$-pattern discrepancy between large and small graphs, then there are "bad" global minima that fail to generalize to larger graphs.

\textbf{(iii) GNNs converge to non-generalizing solutions.} 
In Section \ref{sec: size gen problem empirical validation} we complement these theoretical results with a controlled empirical study that investigates the generalization capabilities of the solutions that GNNs converge to. 
We show, for several synthetic graph distributions in which we have control over the graph structure, that the generalization of GNNs in practice is correlated with the discrepancy between the local distributions of large and small graphs. Specifically, {when} the $d$-patterns in large graphs are not found in small graphs, GNNs tend to converge to a global minimum that succeeds on small graphs and fail on large graphs. This happens even if there is a "good" global minimum that solves the task for all graph sizes. This phenomenon is also prevalent in real datasets as we show in Section \ref{sec:improve size gen}.

\textbf{(iv) Size generalization can be improved.}
Lastly, In Section \ref{sec:improve size gen}, we discuss two approaches for improving size generalization, motivated by our findings. We first formulate the learning problem as a domain adaptation (DA) problem where the source domain consists of small graphs and the target domain consists of large graphs. We then suggest two learning setups: (1) Training  GNNs on a novel self-supervised task aimed at learning meaningful representations for $d$-patterns from both the target and source domains. (2) A semi-supervised learning setup with a limited number of labeled examples from the target domain.  We show that both setups are useful in a series of experiments on synthetic and real data. Notably, training with our new SSL task increases classification accuracy on large graphs in real datasets.

\section{GNNs and local graph patterns}\label{sec:local graph patterns}

We wish to understand theoretically the conditions where a GNN trained on graphs with a small number of nodes can generalize to graphs with a large number of nodes. To answer this question, we first analyze what information is available to each node after a graph is processed by a $d$-layer GNN. It is easy to see that every node can receive information from its neighbors which are at most $d$ hops away. We note, however, that nodes do not have full information about their $d$-hop neighborhood. For example, GNNs cannot determine if a triangle is present in a neighborhood of a given node~\cite{chen2020can}.

To characterize the information that can be found in each node after a $d$-layer GNN, we introduce the notion of $d$-patterns, motivated by the structure of the node descriptors used in the Weisfeiler-Lehman test \cite{weisfeiler1968reduction}: a graph isomorphism test which was recently shown to have the same representational power as GNNs (\cite{xu2018powerful, morris2019weisfeiler}). 

\begin{definition}[$d$-patterns]
Let $C$ be a finite set of node features, and let $G=(V,E)$ be a graph with node feature $c_v\in C$ for every node $v\in V$. We define the \textbf{d-pattern} of a node $v\in V$ for $d\geq 0$ recursively: For $d=0$, the $0$-pattern is $c_v$. For $d>0$, the $d$-pattern of $v$ is $p=(p_v,\{(p_{i_1}, m_{p_{i_1}}), \dots, (p_{i_\ell}, m_{p_{i_\ell}})\})$ iff node $v$ has $(d-1)$-pattern $p_v$ and for every $j\in \{1,\dots,\ell\}$ the number of neighbors of $v$ with $(d-1)$-pattern $p_{i_j}$ is exactly $m_{p_{i_j}}$. Here, $\ell$ is the number of distinct neighboring $d-1$ patterns of $v$.
\end{definition}



\begin{figure}[t]
\centering
    \includegraphics[width=0.4\textwidth]{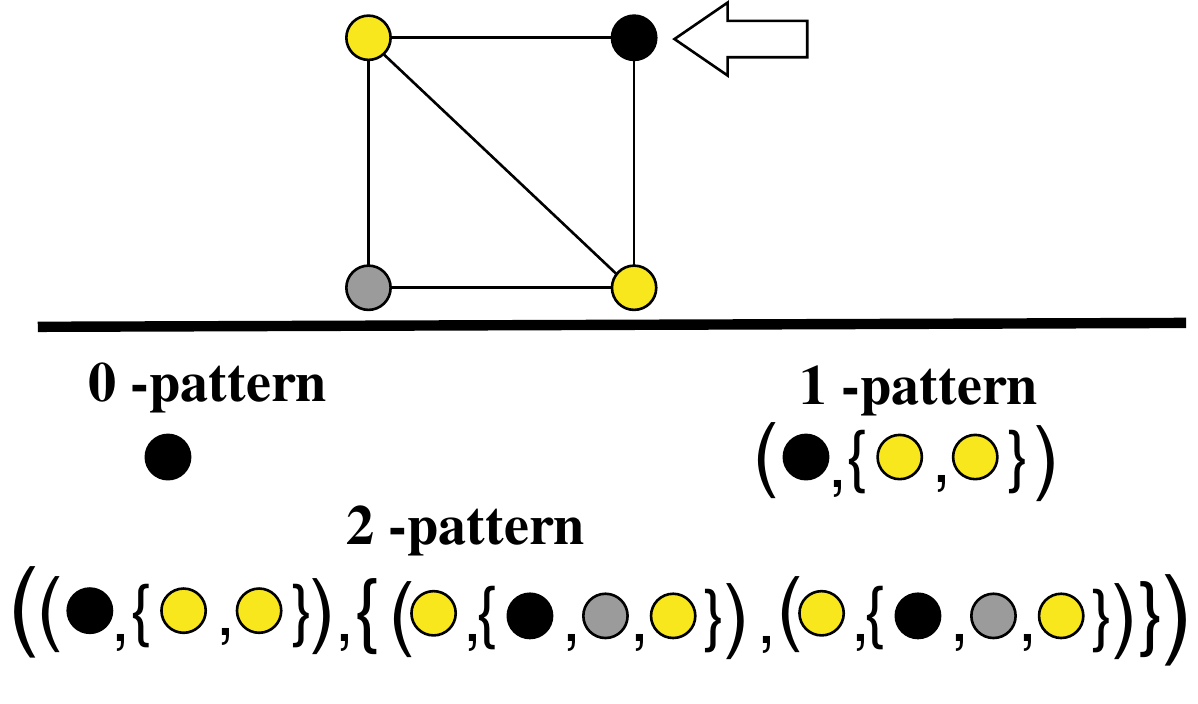}
    \caption{\textbf{Top:} A graph with 4 nodes. Each color represent a different feature. \textbf{Bottom:} The 0,1 and 2-patterns of the black node.}
    \label{fig:0-1-2-patterns}
\end{figure}
In other words, the $d$-pattern of a node is an encoding of the $(d-1)$-patterns of itself and its neighbors.  
For example, assume all the nodes in the graphs start with the same node feature. The $1$-pattern of each node is its degree. The $2$-pattern of each node is for each possible degree $i\in \mathbb{N}$ the number of neighbors with degree $i$, concatenated with the degree of the current node. In the same manner, the $3$-pattern of a node is for each possible $2$-pattern, the number of its neighbors with this exact $2$-pattern. 

\revision{
\figref{fig:0-1-2-patterns}.
illustrates  $0,1$ and $2$-patterns for a graph with three categorical node features, represented by three colors (yellow, grey, and black). For this case, which generalizes the uniform node feature case discussed above, the 0-pattern is the node's categorical feature; 1-patterns count the number of neighbors with a particular feature. The same definition applies to high-order patterns.}

We claim that the definition of $d$-patterns gives an exact characterization to the potential knowledge that a $d$-layer GNN has on each node. First, Theorem \ref{thm:d-patterns constant} is a restatement of Theorem 1 in \cite{morris2019weisfeiler} in terms of $d$-patterns:

\begin{theorem}\label{thm:d-patterns constant}
Any function that can be represented by a $d$-layer GNN is constant on nodes with the same $d$-patterns. 
\end{theorem}

The theorem states that any $d$-layer GNN will output the same result for nodes with the same $d$-pattern. Thus, we can refer to the output of a GNN on the $d$-patterns themselves. 
We stress that these $d$-patterns only contain a part of the information regarding the $d$-neighborhood ($d$ hops away from the node), and different neighborhoods could have the same $d$-patterns. The full proof can be found in \appref{appen:proofs from theoretical results} and follows directly from the analogy between the iterations of the WL algorithm and $d$-patterns. 

Next, the following theorem shows that given a set of $d$-patterns and the desired output for each such pattern, there is an assignment of weights to a GNN with $d+2$ layers that perfectly fits the output for each pattern. 


\begin{theorem}\label{thm:overfit}
Let $C$ be a finite set of node features, $P$ be a finite set of $d$-patterns on graphs with maximal degree $N\in\mathbb{N}$, and for each pattern $p\in P$ let $y_p\in\mathbb{R}$ be some target label. Then there exists a GNN with $d+2$ layers, width bounded by $\max\left\{(N+1)^d \cdot |C|, 2\sqrt{|P|} \right\}$ and ReLU activation such that for every graph $G$ with nodes $v_1,\dots,v_n$ and corresponding $d$-patterns $p_1,\dots,p_n \subseteq P$, the output of this GNN on  $v_i$ is exactly $y_{p_i}$.
\end{theorem}
The full proof is in \appref{appen:proofs from theoretical results}. This theorem strengthens Theorem 2 from \cite{morris2019weisfeiler} in two ways: (1) We prove that one can specify the output for every $d$-pattern while \cite{morris2019weisfeiler} show that there is a $d$-layer GNN that can distinguish all $d$-patterns; (2) Our network construction is more efficient in terms of width and dependence on the number of $d$-patterns ($2\sqrt{|P|}$ instead of $|P|$).

We note that the width of the required GNN from the theorem is not very large if $d$ is small, where $d$ represents the depth of the GNN. In practice, shallow GNNs are very commonly used and are empirically successful. The $d+2$ layers in the theorem can be split into $d$ message-passing layers plus $2$ fully connected layers that are applied to each node independently. 
\thmref{thm:overfit} can be readily extended to a vector output for each $d$-pattern, at the cost of increasing the width of the layers.

Combining \thmref{thm:d-patterns constant} and \thmref{thm:overfit} shows that we can independently control the values of $d$-layer GNNs on the set of $d$-patterns (possibly with an additional node-wise function) and these values completely determine the GNN's output.

\section{"Bad" global minima exist}\label{sec:bad_global}

We now consider \emph{any} graph-prediction task solvable by a $d$-layer GNN. Assume we have a training distribution of (say, small) graphs and a possibly different test distribution of (say, large) graphs. We show that if the graphs in the test distribution introduce unseen $d$-patterns, then there exists a $(d+3)$-layer GNN that solves the task on the train distribution and fails on the test distribution. We will consider both graph-level tasks (i.e. predicting a single value for the entire graph, e.g., graph classification) and node-level tasks (i.e. predicting a single value for each node, e.g., node classification).

\begin{theorem}\label{thm:graph tasks}
Let $P_1$ and $P_2$ be finitely supported distributions of graphs. Let $P^d_1$ be the distribution of $d$-patterns over $P_1$ and similarly  $P^d_2$ for $P_2$. Assume that any graph in $P_2$ contains a node with a $d$-pattern in  $P^{d}_2\setminus P^{d}_1$.  Then, for any graph regression task solvable by a GNN with depth $d$ there exists a GNN with depth at most $d + 3$ that perfectly solves the task on $P_1$ and predicts an answer with arbitrarily large error on all graphs from $P_2$.
\end{theorem}


The proof directly uses the construction from \thmref{thm:overfit}, and can be found in \appref{appen:proof from sec corollaries}. The main idea is to leverage the unseen $d$-patterns from $P_2^d$ to change the output on graphs from $P_2$. 

As an example, consider the task of counting the number of edges in the graph. In this case, there is a simple GNN that generalizes to all graph sizes: the GNN first calculates the node degree for each node using the first message-passing layer and then uses the readout function to sum the node outputs. This results in the output $2|E|$, which can be scaled appropriately. To define a network that outputs wrong answers on large graphs under our assumptions, we can use \thmref{thm:overfit} and make sure that the network outputs the node degree on patterns in $P^{d}_1$ and some other value on patterns in $P^{d}_2\setminus P^{d}_1$. 

Note that although we only showed in \thmref{thm:overfit} that the output of GNNs can be chosen for nodes, the value of GNNs on the nodes has a direct effect on graph-level tasks. This happens because of the global readout function used in GNNs, which aggregates the GNN output  over all the nodes. 

Next, we prove a similar theorem for node tasks. Here, we show a relation between the discrepancy of $d$-pattern distributions  and the error on the large graphs.

\begin{theorem}\label{thm:overfit size node}
Let $P_1$ and $P_2$ be finitely supported distributions on graphs, and 
let $P^{d}_1$ be the distribution of $d$-patterns over $P_1$ and similarly  $P^{d}_2$ for $P_2$. 
For any node prediction task  which is solvable by a GNN with depth $d$ and $\epsilon>0$ there exists a GNN with depth at most $d + 2$ that has 0-1 loss (averaged over the nodes) smaller then $\epsilon$ on $P_1$ and 0-1 loss $\Delta(\epsilon)$ on $P_2$, where $    \Delta(\epsilon)=\max_{A:P^{d}_1(A)<\epsilon}P^{d}_2(A).$
Here, $A$ is a set of $d$-patterns, and $P(A)$ is the total probability mass for that set under $P$.
\end{theorem}

This theorem shows that for node prediction tasks, if there is a large discrepancy between the graph distributions (a set of $d$-patterns with small probability in $P_1^d$ and large probability in $P_2^d$), then there is a solution that solves the task on $P_1$, and generalizes badly for $P_2$. The full proof can be found in \appref{appen:proof from sec corollaries}.

\textbf{Examples.}
The above results show that even for simple tasks, GNNs may fail to generalize to unseen sizes, here are two examples. (i) Consider the task of counting the number of edges in a graph. From \thmref{thm:graph tasks} there is a GNN that successfully outputs the number of edges in graphs with max degree up to $N$, and fails on graphs with larger max degrees. 
(ii) Consider some node regression task, when the training set consists of graphs sampled i.i.d from an Erd\H{o}s-R\'enyi graph $G(n,p)$, and the test set contains graphs sampled i.i.d from $G(2n,p)$. In this case, a GNN trained on the training set will be trained on graphs with an average degree  $np$, while the test set contains graphs with an average degree $2np$. When $n$ is large, with a very high probability, the train and test set will not have any common $d$-patterns, for any $d>0$. Hence, by \thmref{thm:overfit size node} there is a GNN that solves the task for small graphs and fails on large graphs. 

The next section studies the relation between size generalization and local graph structure in controlled experimental settings on synthetic data. 

\section{A controlled empirical study}\label{sec: size gen problem empirical validation}

The previous section showed that there exist bad global minima that fail to generalize to larger graphs. In this section, we study empirically whether common training procedures lead to bad global minima in practice. 
Specifically, we demonstrate on several synthetic graph distributions, that reaching bad global minima is tightly connected to the discrepancy of $d$-pattern distributions between large and small graphs. We identify two main phenomena:
\textbf{(A)} When there is a large discrepancy between the $d$-pattern distributions of large and small graphs, GNNs fail to generalize; \textbf{(B)} As the discrepancy between these distributions gets smaller, GNNs get better at generalizing to larger graphs.

\textbf{Tasks.} In the following experiments, we use a controlled regression task in a student-teacher setting. In this setting, we sample a ``teacher" GNN with random weights (drawn i.i.d from $U([-0.1,0.1]))$, freeze the network, and label each graph in the dataset using the output of the ``teacher" network. Our goal is to train a ``student" network, which has the same architecture as the ``teacher" network, to fit the labels of the teacher network. The advantages of this setting are twofold: (1) \emph{A solution is guaranteed to exist}:  We know that there is a weight assignment of the student network which perfectly solves the task for graphs of any size.
(2) \emph{Generality}: It covers a diverse set of tasks solvable by GNNs. 
As the evaluation criterion, we use the squared loss.

\begin{figure*}[ht]
    \centering
    \begin{tabular}{cccc}
         \includegraphics[width=0.23\textwidth]{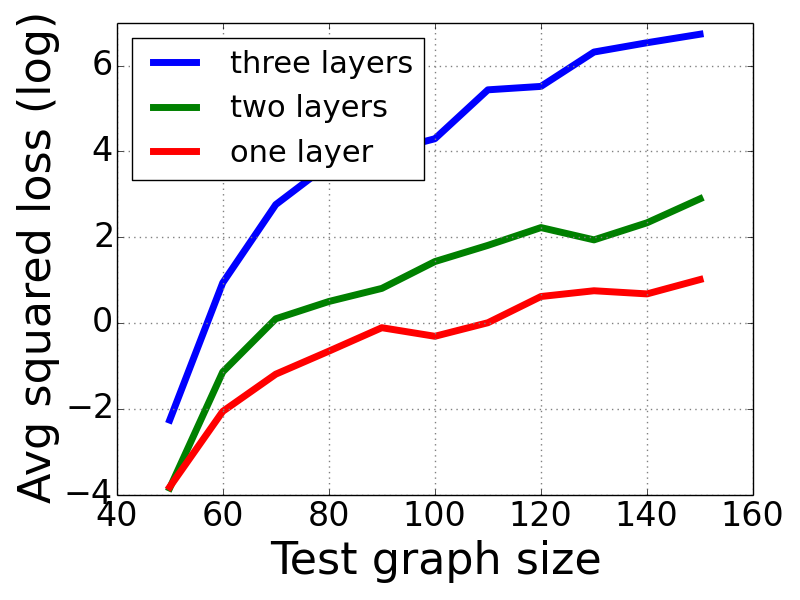}& \includegraphics[width=0.23\textwidth]{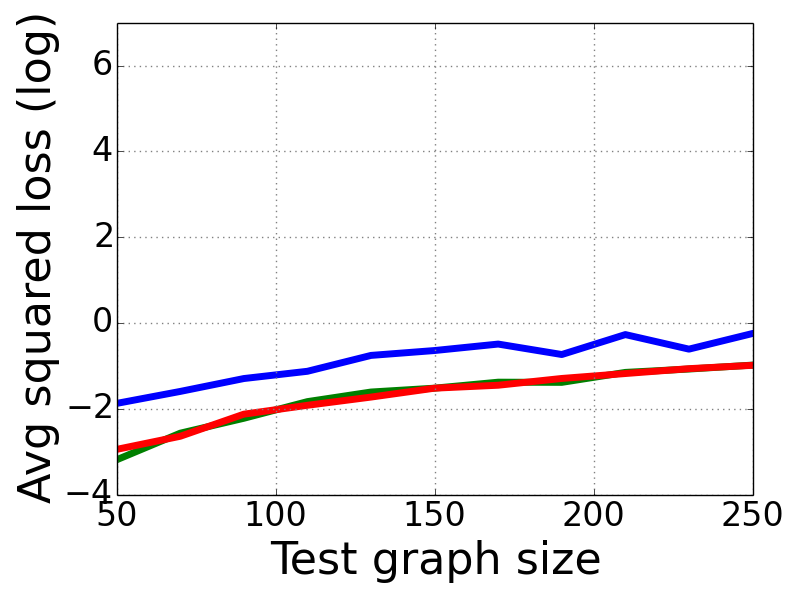}& \includegraphics[width=0.23\textwidth]{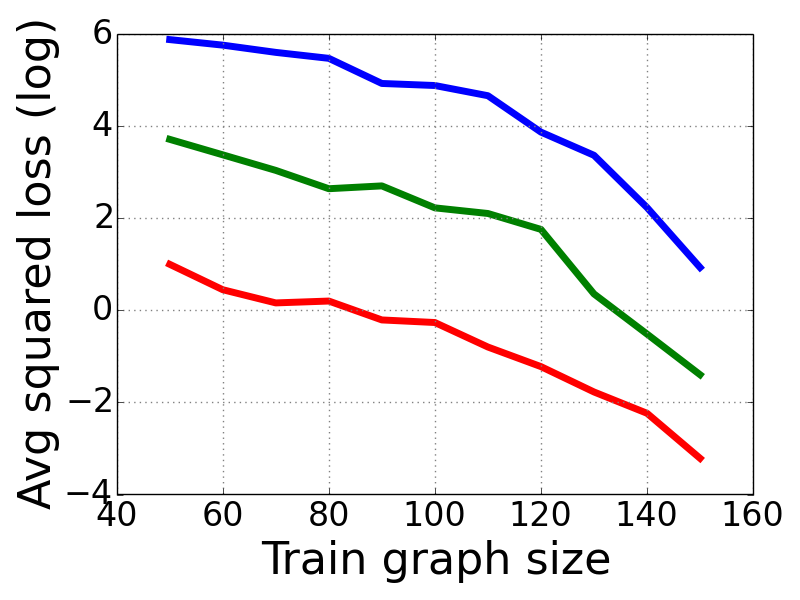}&  \includegraphics[width=0.23\textwidth]{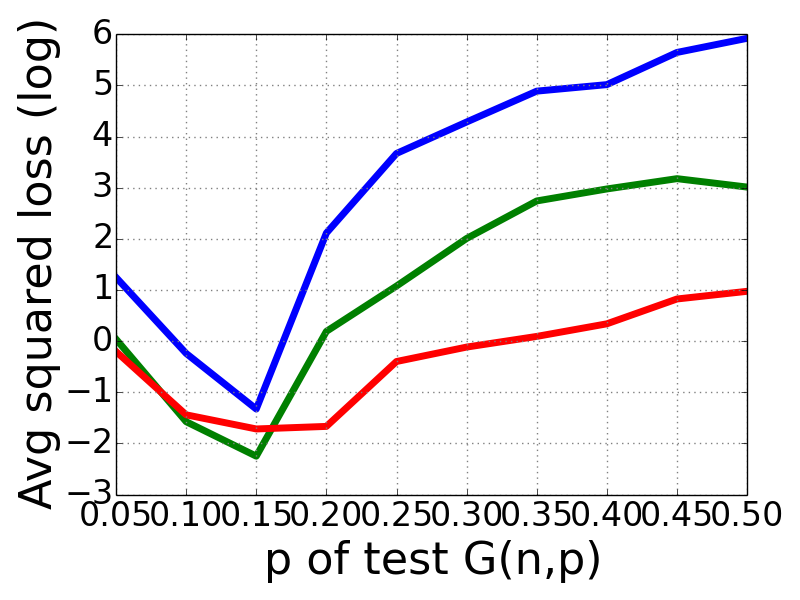}
         \\
         (a)&(b)&(c)&(d) 
    \end{tabular}

    \caption{ The effect of graph size and $d$-pattern distribution on generalization in $G(n,p)$ graphs in a student-teacher graph regression task. The $y$-axis represents the squared loss in $\log_{10}$  scale. (a) Bounded training size $n\in [40,50]$ and varying test size with constant $p=0.3$ (b) \revision{Bounded training training size $n\in [40,50]$ and varying test size {while keeping node degrees constant by changing} 
    $p\in[0.15,0.3]$} 
    . (c) Varying train size with constant test size.  We train on graphs with $n$ nodes and constant $p=0.3$. Here, $n$ is drawn uniformly from $[40,x]$ and $x$ varies; test on $n=150$, $p=0.3$. (d) Train on $n$ drawn uniformly from $[40,50]$ and $p=0.3$ test on $n=100$ and varying $p$. See discussion in the text.} 
    \label{fig:validation_plots}
\end{figure*}

\textbf{Graph distribution.} Graphs were drawn from a $G(n,p)$ distribution. This distribution is useful for testing our hypothesis since we can modify the distribution of $d$-patterns simply by changing either $p$ or $n$. For example, 1-patterns represent node degrees, and in this model, the average degree of graphs generated from $G(n,p)$ is $np$. We provide experiments on additional graph distributions like PA in Appendix \ref{append:addtional experiments size gen problem}.

\textbf{Architecture and training protocol.}  We use a GNN as defined in \citep{morris2019weisfeiler} with ReLU activations. The number of GNN layers in the network we use is either $1,2$ or $3$; the width of the teacher network is $32$ and of the student network $64$, providing more expressive power to the student network. We obtained similar results when testing with a width of 32, the same as the teacher network. We use a summation readout function followed by a two-layer fully connected suffix. We use ADAM with a learning rate of $10^{-3}$. We added weight decay ($L_2$ regularization) with $\lambda = 0.1$. We performed a hyper-parameters search on the learning rate and weight decay and use validation-based early stopping on the source domain (small graphs).  The results are averaged over 10  random seeds. We used Pytorch Geometric \citep{fey2019fast} on NVIDIA DGX-1.

\textbf{Experiments} We conducted four experiments, shown in Figure \ref{fig:validation_plots} (a-d). We note that in all the experiments, the loss on the validation set was effectively zero. First, we study the generalization of GNNs by training on a bounded size range $n\in [40,50]$ and varying the test size in $[50,150]$. 

Figure \ref{fig:validation_plots} (a) shows that when $p$ is kept constant while increasing the test graph sizes, size generalization degrades. Indeed, in this case, the underlying $d$-pattern distribution diverges from the training distribution.  \revision{In \appref{append:addtional experiments size gen problem} we demonstrate that this problem persists to larger graphs with up to 500 nodes.}

On the flip side, Figure \ref{fig:validation_plots} (b) shows that when $p$ is properly normalized to keep the degree $np$ constant while varying the graph size then we have significantly better generalization to large graphs. In this case, the $d$-pattern distribution remains similar.

In the next experiment, shown in Figure \ref{fig:validation_plots} (c) we keep the test size constant $n=150$ and vary the training size $n\in[40,x]$ where $x$ varies in $[50,150]$ and $p=0.3$ remains constant. In this case we can see that as we train on graph sizes that approach the test graph sizes, the $d$-pattern discrepancy reduces and generalization improves. 

In our last experiment shown in Figure \ref{fig:validation_plots} (d), we train on $n\in[40,50]$ and $p=0.3$ and test on $G(n,p)$ graphs with $n=100$ and $p$ varying from $0.05$ to $0.5$. As mentioned before, the expected node degree of the graphs is $np$, hence the distribution of $d$-patterns is most similar to the one observed in the training set when $p=0.15$. Indeed, this is the value of $p$ where the test loss is minimized.

\textbf{Conclusions.} 
First, our experiments confirm phenomena \textbf{(A-B)}. 
Another conclusion is that size generalization is more difficult when using deeper networks. This is consistent with our theory since in these cases the pattern discrepancy becomes more severe: for example, $2$-patterns divide nodes into significantly more $d$-pattern classes than $1$-patterns. Further results on real datasets appear in \secref{sec:improve size gen}.

\textbf{Additional experiments.} in  \appref{append:addtional experiments size gen problem},  We show that the conclusions above are consistent along different tasks (max clique, edge count, node regression), distributions (PA and point cloud graphs), and architectures (GIN \citep{xu2018powerful}). We also tried other activation functions (tanh and sigmoid). \revision{Additionally, we experimented with generalization from large to small graphs. Our previous understanding is confirmed by the findings of the present experiment: generalization is better when the training and test sets have similar graph sizes (and similar $d$-pattern distribution).} 
\section{Towards improving size generalization}\label{sec:improve size gen}
The results from the previous sections imply that the problem of size generalization is not only related to the size of the graph in terms of the number of nodes or edges but to the distribution of $d$-patterns. Based on this observation, we now formulate the size-generalization problem as a domain adaptation (DA) problem. We consider a setting where we are given two distributions over graphs: a source distribution \DS{} (say, for small graphs) and a target distribution \DT{} (say, for large graphs). The main idea is to adapt the network to unseen $d$-patterns appearing in large graphs.

We first consider the \emph{unsupervised} DA setting, where we have access to labeled samples from the source \DS{} but the target data from \DT{} is unlabeled. Our goal is to infer labels on a test dataset sampled from the target \DT{}. To this end, we devise a novel SSL task that promotes learning informative representations of unseen $d$-patterns. We show that this approach improves the size-generalization ability of GNNs. 

Second, we consider a \emph{semi-supervised} setup, where we also have access to a small number (e.g., 1-10) of labeled examples from the target \DT{}. We show that such a setup, when feasible, can lead to equivalent improvement, and benefits from our SSL task as well.


\subsection{SSL for DA on graphs} In SSL for DA, a model is trained on unlabeled data to learn a \emph{pretext} task, which is different from the main task at hand. If the pretext task is chosen wisely, the model learns useful representations  \citep{doersch2015unsupervised,gidaris2018unsupervised} that can help with the main task. Here, we train the pretext task on both the source and target domains, as was done for images and point clouds  \citep{sun2019unsupervised,achituve2020self}. The idea is that the pretext task aligns the representations of the source and target domains leading to better predictions of the main task for target graphs.

\textbf{Pattern-tree pretext task.} 
We propose a novel pretext task which is motivated by   sections \ref{sec:bad_global}-\ref{sec: size gen problem empirical validation}: one of the main causes for bad generalization is unseen $d$-patterns in the test set. Therefore, we design a pretext task to encourage the network to learn useful representations for these $d$-patterns.

\begin{wrapfigure}[17]{r}{0.25\textwidth}
\centering
\includegraphics[width=0.25\textwidth]{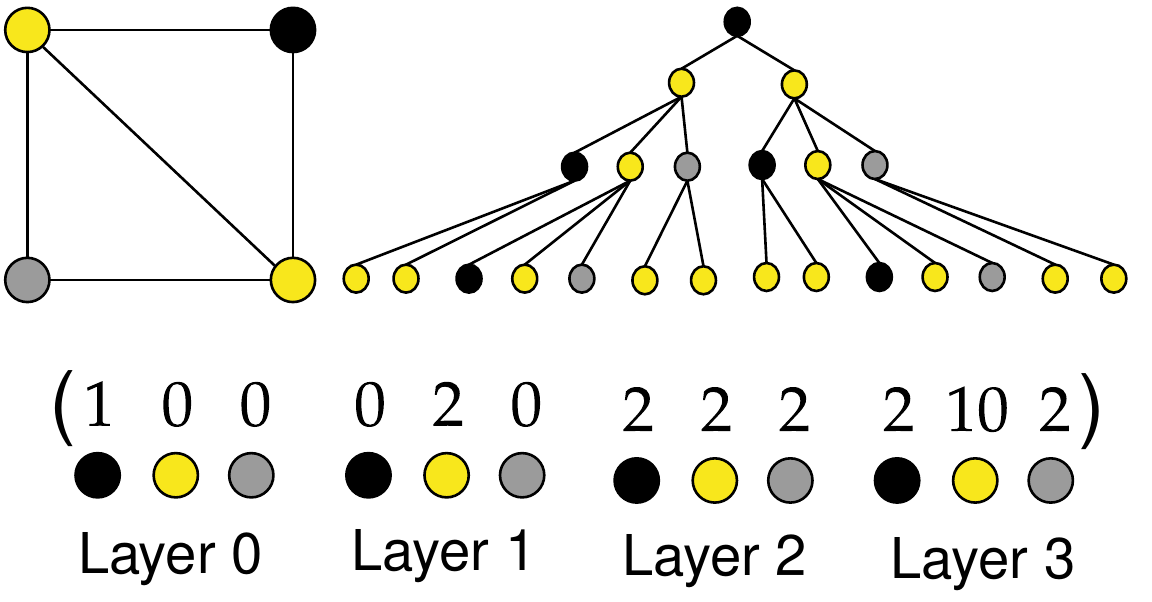}
\caption{\textbf{Top left:} a graph with node features represented by colors. \textbf{Top right:} A tree that represents the $d$-patterns for the black node. \textbf{Bottom: }The tree descriptor is a vector with each coordinate containing the number of nodes from each class in each layer of the tree. \label{fig:pattern tree}} 
\end{wrapfigure}

Our pretext task is a node prediction task in which the output node label is specifically designed to hold important information about the node's $d$-pattern. For an illustration of a label see Figure \ref{fig:pattern tree}.  
The construction of those labels is split into two procedures.

First, we construct a tree that fully represents each node's $d$-pattern. The tree is constructed for a node $v$ in the following way: we start by creating a root node that represents $v$. We then create nodes for all $v$'s neighbors and connect them to the root. All these nodes hold the features of the nodes they represent in the original graph. We continue to grow the tree recursively up to depth $d$ by adding new nodes that represent the neighbors (in the original graph) of the current leaves in our tree.

This is a standard construction, see e.g., \cite{xu2018powerful}. For more details about the construction of the pattern tree see \appref{appen:SSL_tree_task}.

We then calculate a descriptor of the tree that will be used as the SSL output label for each node. The descriptor is a concatenation of histograms of the different node features in each layer of the tree. The network is then trained in a node regression setup with a dedicated SSL head to predict this descriptor.


\subsection{Experiments}

\begin{table*}[t]   
    \setlength{\tabcolsep}{3.5pt}
    \centering
    \begin{sc}
    \scriptsize
    
    \begin{tabular}{ l c c c c c c c c } 
    \textbf{Datasets}     & \textbf{Deezer}  & \textbf{IMDB - B} & \textbf{NCI1}    & \textbf{NCI109} & \textbf{Proteins} & \textbf{Twitch}  & \textbf{DD} & \textbf{Average}
    \\ \hline 
    \textbf{Total-var. distance} & 1 & 0.99 & 0.16 & 0.16 & 0.48 & 1 & 0.15 & -
    \\ \hline 
    \textbf{Small Graphs}  & $56.5 \pm 0.8$ & $63.2 \pm 3.3$ & $75.5 \pm 1.6$ & $78.4 \pm 1.4$ & $75.4 \pm 3.1$ & $69.7 \pm 0.2$ & $71.1 \pm 4.4 $  & 70.0\%                      \\ 
    \hline 
    \textbf{Vanilla}  & $41.1 \pm 6.8$ & $55.9 \pm 7.8$ & $65.9 \pm 4.3$ & $68.9 \pm 3.8$ & $76.0 \pm 8.5$ & $60.5 \pm 3.6$  & $76.3 \pm 3.2 $ & 63.5\%                      \\ 
        \hline 

    \textbf{Homo-GNN}      & $40.5 \pm 6.6$ & $56.3 \pm 7.0$ & $66.0 \pm 3.7$ & $68.8 \pm 3.2$ & $77.1 \pm 10.0$ & $60.8 \pm   2.3$    & \pmb{$76.8 \pm 3.0$}    & 63.8\% \\ 
    \textbf{NM MTL}   & \pmb{$51.6 \pm 8.5$} & $55.6 \pm 6.8$ & $49.9 \pm 7.8$ & $61.7 \pm 5.7$ & $78.8 \pm 8.4$ & $49.5 \pm 2.8$ & $67.4 \pm 5.4$  & 59.2\%                      \\ 
    \textbf{NM PT}   & $50.1 \pm 7.5$ & $54.9 \pm 6.7$ & $51.7 \pm 6.6$ & $55.8 \pm 5.0$ & $78.2 \pm 8.2$ & $48.4 \pm 4.0$  & $60.3 \pm 15.9$  & 57.1\%                       \\ 
    \textbf{GAE MTL}  & $49.4 \pm 11.0$ & $55.5 \pm 6.0$ & $51.2 \pm 9.9$ & $57.6 \pm 9.4$ & $79.5 \pm 11.7$ & $62.5 \pm 5.1$  & $67.8 \pm 10.0$  & 60.5\%                      \\ 
    \textbf{GAE PT}  & $47.1 \pm 10.0$ & $54.1 \pm 6.8$ & $58.9 \pm 7.6$ & $67.2 \pm 5.6$ & $70.5 \pm 9.4$ & $53.6 \pm 4.7$  & $69 \pm 7.1$  & 60.1\%                      \\ 
    \textbf{NML MTL}& $46.4 \pm 9.5$ & $54.4 \pm 7.0$ & $52.3 \pm 6.3$ & $56.2 \pm 6.5$ & $78.7 \pm 6.8$ & $57.4 \pm 4.1$  & $64.7 \pm 11.9$  & 58.6\%                      \\ 
    \textbf{NML PT}  & $48.4 \pm 10.7$ & $53.8 \pm 6.1$ & $54.6 \pm 6.2$ & $56.1 \pm 8.1$ & $76.3 \pm 8.0$  & $54.9 \pm 4.7$ & $61.4 \pm 15.1$ & 57.9\%   
    \\
    \textbf{CL MTL} & $48.2 \pm 10.9$ & $54.6 \pm 6.6$ & $52.2 \pm 6.8$ & $55.7 \pm 5.8$ & $76.6 \pm 7.7$ & $59.4 \pm 3.5$ & $63.6 \pm 15.0$ & $58.6\%$ \\
    \textbf{CL PT} & $47.6 \pm 9.7$ & $53.6 \pm 7.5$ & $57.4 \pm 8.1$ & $57.3 \pm 6.1$ & $77.6 \pm 4.7$ & $53.9 \pm 7.1$ & $69.2 \pm 5.5$ & $59.5\%$
    \\ \hline
    \textbf{Pattern MTL (ours)} & $45.6 \pm 8.8$ & $56.8 \pm 9.2$ & $60.5 \pm 7.5$  & $67.9 \pm 7.2$ & $75.8 \pm 11.1$ & $61.6 \pm 3.5$ & \pmb{$76.8 \pm 3.0$}   & 63.6\%                      \\ 
    \textbf{Pattern PT (ours)}  & $44.0 \pm 7.7$ & \pmb{$61.9 \pm 3.2$} & \pmb{$67.8 \pm 11.7$} & \pmb{$74.8 \pm 5.7$} & \pmb{$84.7 \pm 5.1$} & \pmb{$64.5 \pm 3.3$} & $74.9 \pm 5.2$ & \textbf{67.5\%}             \\ \hline
    \end{tabular}
    \end{sc}
    \caption{Test accuracy  of compared methods in 7 binary classification tasks. The Pattern tree method with pretraining achieves the highest accuracy in most tasks and increases the average accuracy from 63\% to 67\% compared with the second-best method. High variance is due to the domain shift between the source and target domain.  
    }
    \label{tab:real datasets}
\end{table*}

\textbf{Baselines.} We compare our new pretext task to the following baselines: (1) \textbf{Vanilla}: standard training on the source domain;
(2) \textbf{HomoGNN}  \citep{tang2020towards}  a homogeneous GNN without the bias term trained on the source domain; (3) \textbf{Graph autoencoder (GAE)} pretext task \citep{kipf2016variational}; (4) \textbf{Node masking (NM)} pretext task from \cite{hu2019strategies} where at each training iteration we mask $10\%$ of the node features and the goal is to reconstruct them. In case the graph does not have node features then the task was to predict the degree of the masked nodes. (5) \textbf{Node metric learning (NML)}: we use metric learning to learn useful node representations. We use a corruption function that given a graph and corruption parameter $p\in[0,1]$, replaces $p|E|$ of the edges with random edges, and thus can generate positive ($p=0.1$) and negative ($p=0.3$) examples for all nodes of the graph. We train with the triplet loss \citep{weinberger2009distance}. \revision{(6) \textbf{Contrastive learning (CL)}: In each iteration, we obtain two similar versions of each graph, which are used to compute a contrastive loss \cite{qiu2020gcc,you2020graph} against other graphs. We follow the protocol of \cite{you2020graph}, using a corruption function of edge perturbation that randomly adds and removes $5\%$ of the edges in the graph.}

\textbf{Datasets.}
We use datasets from \cite{Morris+2020} and \cite{karateclub} (Twitch egos and Deezer egos). We selected datasets that have a sufficient number of graphs (more than 1,000) and with a non-trivial split to small and large graphs as detailed in \appref{appen:datasets statistics}. In total, we used 7 datasets, 4 from molecular biology (NCI1, NCI109, D\&D, Proteins), and $3$ from social networks (Twitch ego nets, Deezer ego nets, IMDB-Binary). In all datasets, $50\%$ smallest graphs were assigned to the training set, and the largest $10\%$ of graphs were assigned to the test set. We further split a random $10\%$ of the small graphs as a validation set.
\begin{figure}[t]
    \centering
\includegraphics[width=0.95\linewidth]{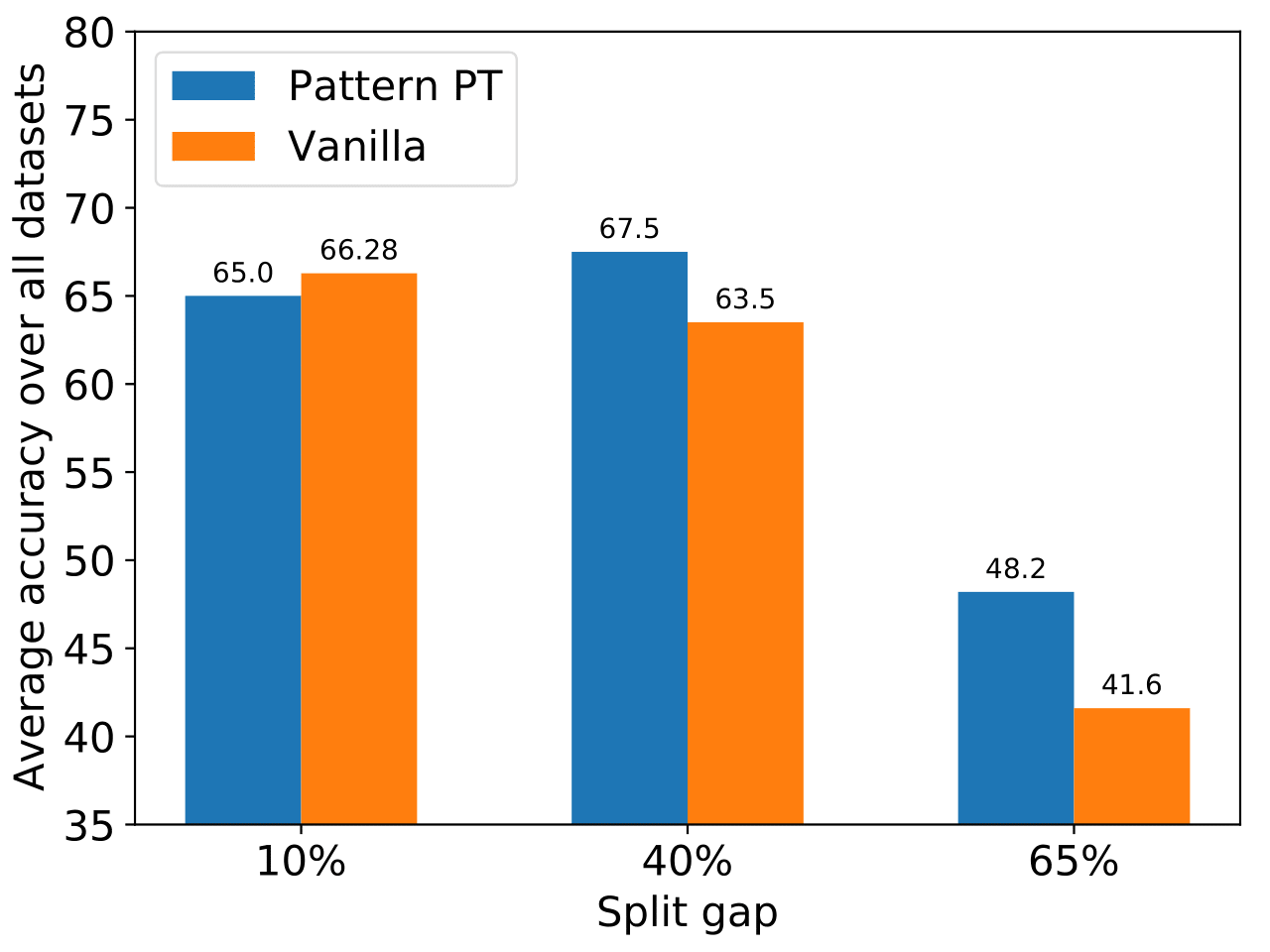}
\caption{ \revision{Average {accuracy} on different size splits in the unsupervised setup for (i) $d$-pattern pretraining and (ii) no SSL (Vanilla). {Accuracy is} averaged over all the datasets in table \ref{tab:real datasets}. }} 
\label{fig:diff_splits}
\end{figure}

\textbf{Architecture and training protocol.} The setup is the same as in \secref{sec: size gen problem empirical validation} with a three-layer GNN in all experiments. Given a pretext task we consider two different training procedures: (1) \textbf{Multi-task learning (MTL)} \citep{you2020does}; (2) \textbf{Pretraining (PT)} \citep{hu2019strategies}. For MTL we use equal weights for the main and SSL tasks. In the semi-supervised setup, we used equal weights for the source and target data. More details on the training procedures and the losses can be found in \appref{appen:training procedure}. 


\textbf{$d$-pattern distribution in real datasets.} In \appref{appen:counting patterns} we study the discrepancy between the local patterns between small and large graphs on all the datasets mentioned above. The second row of Table \ref{tab:real datasets} summarizes our findings with the total variation ($TV$) distances between $d$-pattern distributions of small and large graphs. The difference between these distributions is severe for all social network datasets ($TV\approx 1$), and milder for biological datasets ($TV \in [0.15,0.48]$). 

Next, we will see that a discrepancy between the $d$-patterns leads to bad generalization and that correctly representing the patterns of the test set improves performance. 


\textbf{Results for unsupervised DA setup.} 
Table \ref{tab:real datasets} compares the effect of using the Pattern-tree pretext task to the baselines described above. The \emph{small graphs} row presents vanilla results on a validation set with small graphs for comparison. The small graph accuracy on 5 out of 7 datasets is larger by 7.3\%-15.5\% than on large graphs, indicating that the size-generalization problem is indeed prevalent in real datasets.

Pretraining with the $d$-patterns pretext task outperforms other baselines in 5 out 7 datasets, with an average $4\%$ improved accuracy on all datasets.  HOMO-GNN slightly improves over the vanilla while other pretext tasks do not improve average accuracy. Specifically, for the datasets with high discrepancy of local patterns (namely, IMDB, Deezer, Proteins, and Twitch), pretraining with our SSL task improves nicely over vanilla training (by $5.4\%$ on average). Naturally, the accuracy here is lower than SOTA on these datasets because the domain shift makes the problem harder.

\revision{\figref{fig:diff_splits} {shows two additional experiments, conducted on all  datasets using different size splits. First,  using a gap of 65\% (training on the 30\% smallest graphs and testing on the 5\% largest graphs), and second, using} a gap of 10\% (training on the 50\% smallest graphs and testing on graphs in the 60-70-percentile). The results are as expected: (1) When training without SSL, larger size gaps hurt more (2) SSL improves over Vanilla training with larger gaps.}

\textbf{Results for semi-supervised DA setup.} Figure \ref{fig:few_shot_vanilla_vs_pattern} compares the performance of vanilla training versus pretraining with the pattern-tree pretext task in the semi-supervised setup. As expected, the accuracy monotonically increases with respect to the number of labeled examples in both cases. Still, we would like to highlight the improvement we get by training on only a handful of extra examples. Pretraining with the pretext task yields better results in the case of 0,1,5 labeled examples and comparable results with 10 labeled examples.

\textbf{Additional experiments} We provide additional experiments on the synthetic tasks discussed in \secref{sec: size gen problem empirical validation} in Appendix \ref{sec:more experiments from sec solution}. We show that the pattern-tree pretext task improves generalization in the student-teacher setting (while not solving the edge count or degree prediction tasks). In addition, adding even a single labeled sample from the target distribution significantly improves performance. \revision{We additionally tested our SSL task on a combinatorial optimization problem of finding the max clique size in the graph, our SSL improves over vanilla training by a factor of 2, although not completely solving the problem. Also, we tested on several tasks from the "ogbg-molpcba" dataset (see \cite{hu2020open}), although the results are inconclusive. This is further discussed in \secref{sec:discussion}.}

\begin{figure}[t]
\centering
    \includegraphics[width=0.5\textwidth]{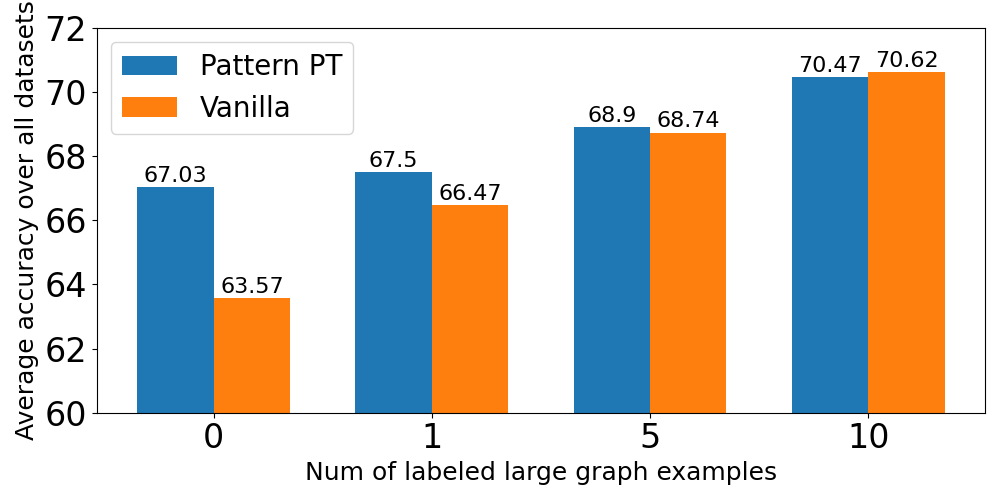}
    \caption{Average classification results in the semi-supervised setup  for (i) $d$-pattern pretraining and (ii) no SSL (Vanilla). Results were averaged over all the datasets in table \ref{tab:real datasets}. }
    \label{fig:few_shot_vanilla_vs_pattern}
\end{figure}



\section{Related work}

\textbf{Size generalization.} Several papers observed successful generalization across graph sizes, but the underlying reasons were not investigated~\citep{li2018combinatorial, maron2018invariant, luz2020learning}. More recently, \citep{velivckovic2019neural} showed that when training GNNs to perform simple graph algorithms step by step they generalize better to graphs of different sizes. Unfortunately, such training procedures cannot be easily applied to general tasks. \cite{knyazev2019understanding} studied the relationship between generalization and attention mechanisms. \cite{bevilacqua2021on} study graph extrapolation using causal modeling. On the more practical side, \cite{joshi2019efficient, joshi2020learning, khalil2017learning}, study the Traveling Salesman Problem (TSP), and show empirically that size generalization on this problem is hard. \cite{corso2020principal} study several multitask learning problems on graphs and evaluate how the performance changes as the size of the graphs change. 
In another line of work,  \citet{tang2020towards, nachmani2020molecule} considered adaptive depth GNNs. 
In our paper, we focus on the predominant GNN architecture with a fixed number of message-passing layers.
Several works also studied size generalization and expressivity when learning set-structured inputs \citep{zweig2020functional, bueno2020limitations}. 
In \cite{santoro2018measuring} the authors study generalization in abstract reasoning.  

\textbf{Generalization in graph neural networks.} 
Several works studied generalization bounds for certain classes of GNNs \citep{garg2020generalization, puny2020graph,verma2019stability, liao2020pac, du2019graph}, but did not discuss size generalization. \cite{sinha2020evaluating} proposed a benchmark for assessing the logical generalization abilities of GNNs. 

\textbf{self-supervised and unsupervised learning on graphs.} 
One of the first papers to propose an unsupervised learning approach for graphs is \cite{kipf2016variational}, which resulted in several subsequent works \citep{park2019symmetric, salha2019keep}. \citep{velickovic2019deep} suggested an unsupervised learning approach based on predicting global graph properties from local node descriptors. \citep{hu2019strategies} suggested several unsupervised learning tasks that can be used for pretraining. More recently, \citep{jin2020self, you2020does} proposed several self-supervised tasks on graphs, such as node masking. These works mainly focused on a single graph learning setup.  
\revision{
\cite{you2020graph,qiu2020gcc} applied contrastive learning techniques for unsupervised representation learning on graphs. The main difference between our SSL task and contrastive learning is that following our theoretical observation, our SSL task focuses on representing the local structure of each node, rather than a representation that takes into account the entire graph.  
}

\section{Conclusion and Discussion}\label{sec:discussion} 

This work is a step towards gaining an understanding of the size-generalization problem in graph neural networks. We showed that for important graph distributions, GNNs do not naturally generalize to larger graphs even on simple tasks.  We started by defining $d$-patterns, a concept that captures the expressivity of GNNs. We then characterized how the failure to generalize depends on $d$-patterns. Lastly, we suggested two approaches that can improve generalization. Although these approaches are shown to be useful for multiple tasks, there are still some tasks where generalization could not be improved. 

A limitation of our approach is that it assumes categorical node features and bidirectional edges with no features. We plan to expand our approach in the future to address these important use cases.
As a final note, our characterization of $d$-patterns, as well as the methods we proposed, can be applied to other cases where generalization is hindered by distribution shifts, and may also be able to improve results in these situations.




\bibliography{example_paper}
\bibliographystyle{icml2021}

\newpage

\onecolumn

\appendix
\icmltitle{Supplementary Material: From Local Structures to Size Generalization in Graph Neural Networks}

\section{Size generalization in Single-layer GNNs}\label{appen:linear GNN}
We start our discussion on size generalization with a theoretical analysis of a simple setup. We consider a single-layer GNN and an easy task and show that: (1) The training objective has many different solutions, but only a small subset of these solutions generalizes to larger graphs (2) Simple regularization techniques cannot mitigate the problem.

Assume we train on a distribution of graphs. Our task is to predict the number of edges in the graph using a first-order GNN with a single linear layer and additive readout function, for simplicity also consider the squared loss. We first note that the task of edge count can be solved with this architecture for graphs of any size. The intuition is to count the number of neighbors for each node (can be done with a 1-layer GNN), and summing over all nodes using the readout function would give us $2|E|$, where $E$ is the set of edges.

The objective boils down to the following function for any graph $G$ in the training set:
\[
L(w_1,w_2,b;G) = \left(\sum_{u\in V(G)}\left(w_1 \cdot x_u + \sum_{v\in \mathcal{N}(u)} w_2\cdot x_v + b\right) - y\right)^2~.
\]

Here, $G$ is an input graph, $V(G)$ are the nodes of $G$, $\mathcal{N}(v)$ are all the neighbors of node $v$, $w_1,w_2$ and $b$ are the trainable parameters, $y$ is the target and $x_v$ is the node feature for node $v$. Further, assume that we have no additional information on the nodes, so we can just embed each node as a one-dimensional feature vector with a fixed value of $1$. In this simple case, the trainable parameters are also one-dimensional.

For a graph with $n$ nodes and $m$ edges the training objective can also be written in the following form:
\[
L(w_1,w_2,b;G) = \left(nw_1 + 2mw_2 + nb - m\right)^2,\]
One can easily find the solution space, which is an affine subspace defined by $w_2 = \frac{m-n(w_1+b)}{2m} = \frac{1}{2} - \frac{n}{2m}\cdot \frac{w_1+b}{2m}$. In particular, the solutions with $w_1+b=0,~w_2=1/2$  are the only ones which do not depend on the specific training set graph size $n$, and generalize to graphs of any size and with any number of edges.  

It can be readily seen that when training the model on graphs where the ratio $n/m$ between the number of nodes and number of edges is fixed, gradient descent with a small enough learning rate will favor the global solution closest to the initialized point. Hence, by using a standard initialization scheme (e.g. \cite{glorot2010understanding}), with probability 1, the solution that gradient descent converges to is not a generalizing solution. Note that we could train on many graphs with different sizes and still end up in a non-generalizing solution, as long as $n/m$ is fixed. On the other hand, training on graphs with different node/edge ratios necessarily leads to some generalizing solution. This is because the problem is convex, and the generalizing solutions are the only solutions that minimize the loss for graphs with different ratios.

We also note that the generalizing solution ($w_1+b=0,~w_2=1/2$) is not  the least norm solution in general (with respect to both $L_1$ and $L_2$ norms) so simple regularization will not help here (it is the least $L_1$ norm solution if $2m >n$). As we show in \secref{sec: size gen problem empirical validation}, the problem gets worse when considering GNNs with multiple non-linear layers, and this simple solution will not help in this case: we can train deeper GNNs on a wide variety of sizes and the solution will not generalize to other sizes.

\section{Proofs from \secref{sec:local graph patterns}}\label{appen:proofs from theoretical results}

\begin{proof}[Proof of \thmref{thm:d-patterns constant}]
We show that from the definition of $d$-patterns, and the 1-WL algorithm (see \cite{weisfeiler1968reduction}), the color given by the WL algorithm to two nodes is equal iff their $d$-pattern is equal. For the case of $d=0$, it is clear. Suppose it is true for $d-1$, the WL algorithm at iteration $d$ give node $v$ a new color based on the colors given in iteration $d-1$ to the neighbors of $v$. This means, that the color of $v$ at iteration $d$ depends on the multiset of colors at iteration $d-1$ of its neighbors, which by induction is the $(d-1)$-pattern of the neighbors of $v$. To conclude, we use Theorem 1 from \cite{morris2019weisfeiler} which shows that GNNs are constant on the colors of WL, hence also constant on the $d$-patterns.
\end{proof}

To prove \thmref{thm:overfit} we will first need to following claim from \cite{yun2019small} about the memorization power of ReLU networks:

\begin{theorem}\label{thm:small relu memorize}
Let $\{\bx_i,y_i\}_{i=1}^N\in\mathbb{R}^d\times\mathbb{R} $ such that all the $\bx_i$ are distinct and $y_i\in[-1,1]$ for all $i$. Then there exists a $3$-layer fully connected ReLU neural network $f:\mathbb{R}^d\rightarrow\mathbb{R}$ with width $2\sqrt{N}$ such that $f(\bx_i)=y_i$ for every $i$.
\end{theorem}

We will also need the following lemma which will be used in the construction of each layer of the GNN:
\begin{lemma}\label{lem:overfit single node}
Let $N \in\mathbb{N}$ and $f:\mathbb{N}\rightarrow\mathbb{R}$ be a function defined by $f(n) = \bw_2^\top\sigma(\bw_1n - \bb)$ where $\bw_1,\bw_2,\bb\in\mathbb{R}^N$, and $\sigma$ is the ReLU function. Then for every $y_1,\dots,y_N\in\mathbb{R}$ there exists $\bw_1,\bw_2,\bb$ such that $f(n) = y_n$ for $n\leq N$ and $f(n) = (n-N+1)y_N  - (n-N)y_{N-1}$ for $n > N$.
\end{lemma}

\begin{proof}
Define $\bw_1 = \begin{pmatrix} 1 \\ \vdots \\ 1 \end{pmatrix}$, $\bb = \begin{pmatrix} 0 \\ 1 \\ \vdots \\ N-1\end{pmatrix}$. Let $a_i$ be the $i$-th coordinate of $\bw_2$, we will define $a_i$ recursively in the following way: Let $a_1=y_1$, suppose we defined $a_1,\dots,a_{i-1}$, then define $a_i= y_i - 2a_{i-1} - \dots - ia_1$. Now we have for every $n\leq N$:
\begin{align*}
    f(n) = \bw_2^\top\sigma(\bw_1n - \bb) = na_1 + (n-1)a_2 + \dots + a_n = y_n~.
\end{align*}
For $n > N$ we can write $n = N+k$ for $k\geq 1$, then we have:
\begin{align*}
    f(n) &= \bw_2^\top\sigma(\bw_1(N+k) - \bb) = (N+k)a_1 + (N+k-1)a_2 + \dots + (k+1)a_N \\
    &= y_N + k(a_1+a_2 + \dots + a_N) = y_N + k(y_N -a_{N-1} - 2a_{N-2} - \dots - (N-1)a_1) \\
    &= (k+1)y_N  - ky_{N-1}
\end{align*}
\end{proof}

Now we are ready to prove the main theorem:

\begin{proof}[Proof of \thmref{thm:overfit}]
We assume w.l.o.g that at the first iteration each node $i$ is represented as a one-hot vector $\bh^{(0)}_i$ of dimension $|C|$, with its corresponding node feature. Otherwise, since there are $|C|$ node features we can use one GNN layer that ignores all neighbors and only represent each node as a one-hot vector. We construct the first $d$ layers of the $1$-GNN $F$ by induction on $d$. Denote by $a_i=|C|\cdot \left(N^i + N^{i-1}\right)$ the dimension of the $i$-th layer of the GNN for $1\leq i\leq d$, and $a_0 = |C|$.

The mapping has two parts, one takes the neighbors information and maps it into a feature representing the multiset of $d$-patterns, the other part is simply the identity-saving information regarding the d-pattern of the node itself.\\

The $d$ layer structure is 
\begin{equation*}
    \bh_v^{(d)} = U^{(d+1)}\sigma\left(W_2^{(d)}\bh_v^{(d-1)}  + \sum_{u\in\mathcal{N}(v)}W_1^{(d)}\bh_u^{(d-1)} - \bb^{(d)}\right)    
\end{equation*}

We set $W_2^{(d)} =[0,I]^T,\,W_1^{(d)}=[\tilde{W}_1^{(d)T},0]^T$ and $U^{(d+1)}=[\tilde{U}^{(d+1)T},0]^T$ with $\tilde{W}_1^{(d)}\in\R^{Na_{d-1}\times a_{d-1}}$ and $\tilde{U}^{(d+1)}\in\R^{Na_{d-1}\times Na_{d-1}}$. For $\tilde{W}_1^{(d)}$ we set $\bw_i^{(1)}$, its $i$-th row, to be equal to $\be_n$ where $n=\left\lceil\frac{i}{N}\right\rceil$. Let $b^{(1)}_i$ be the $i$-th coordinate of $\bb^{(d)}$, be equal to $i-1~(\text{mod } N)$ for $i\leq N\cdot a_{d-1}$ and zero otherwise. What this does for the first $a_{d-1}$ elements of $W_2^{(d)}\bh_v^{(d)}  + \sum_{u\in\mathcal{N}(v)}W_1^{(d)}\bh_u^{(d)}$ is that each dimension $i$ hold the number of neighbors with a specific (d-1)-pattern. We then replicate this vector N times, and for each replica we subtract a different bias integer ranging from $0$ to $N-1$. To that output we concatenate the original $\bh_v^{(d-1)}$\\

 Next we construct $\tilde{U}^{(d+1)}\in \mathbb{R}^{Na_{d-1}\times Na_{d-1}}$ in the following way: Let $\bu^{(d+1)}_i$ be its $i$-th row, and $u^{(d+1)}_{i,j}$ its $j$-th coordinate. We set $u^{(d+1)}_{i,j}=0$ for every $j$ with $j\neq \left\lceil\frac{i}{N}\right\rceil$ and the rest $N$ coordinates to be equal to the vector $\bw_2$ from \lemref{lem:overfit single node} with labels $y_\ell = 0$ for $\ell \in \{1,\dots,N\}\setminus \{(i\text{ mod } N)+1\}$ and $y_\ell=1$ for $\ell = (i\,\text{ mod } N)+1$.

Using the above construction we encoded the output on node $v$ of the first layer of $F$ as a vector:

This encoding is such that the $i$-th coordinate of $\bh_v^{(d+1)}$ for $1\leq i\leq N\cdot a_{d-1}$ is equal to $1$ iff node $v$ have $(i\text{ mod } N)+1 $ neighbors with node feature $\left\lceil \frac{i}{N}\right\rceil \in \{1,\dots,|C|\}$. The last $a_{d-1}$ rows are a copy of $\bh_v^{(d-1)}$.

\textbf{Construction of the suffix.} Next, we construct the last two layers. First we note that for a node $v$ with $d$-pattern $p$ there is a unique vector $\bz_p$ such that the output of the GNN on node $v$, $\bh^{(d)}_v$, is equal to $\bz_p$. From our previous construction one can reconstruct exactly the (d-1)-pattern of each node, and the exact number of neighbors with each (d-1)-pattern and therefore can recover the d-pattern correctly from the $\bh_v^{(d)}$ embedding. 

Let $y_{\max} := \max_i |y_i|$, and define $\tilde{y}_i = y_i / y_{\max}$.
Finally, we use \thmref{thm:small relu memorize} to construct the $3$-layer fully connected neural network with width at most $2\sqrt{|P|}$ such that for every pattern $p_i\in P$ with corresponding unique representation $\bz_{p_i}$ and label $\tilde{y}_i$, the output of the network on $\bz_{p_i}$ is equal to $\tilde{y}_i$. We construct the last two layers of the GNN such that $W_1^{(d+1)}, W_1^{(d)} =0$, and $W_2^{(d+1)},\bb^{(d+1)},W_2^{(d+2)},\bb^{(d+2)}, W^{(d+3)}$ are the matrices produced from \thmref{thm:small relu memorize}. Note that $W^{(d+3)}$ is the linear output layer constructed from the theorem, thus the final output of the GNN for node $v$ is $W^{(d+3)}\cdot h^{(d+2)}_v$, where $h^{(d+2)}_v$ is the output after $d+2$ layers. We multiply the final linear output layer by $y_{\max}$, such that the output of the entire GNN on pattern $p_i$ is exactly $y_i$.
\end{proof}

\section{Proofs from \secref{sec:bad_global}}\label{appen:proof from sec corollaries}
\begin{proof}[Proof of \thmref{thm:graph tasks}]
    By the assumption, there is a depth $d$ GNN that solves the task for graphs of any size, denote this GNN by $F$. We can write $F$ operating on a graph $G$ as $F(G) = M\left(\sum_{v\in V(G)}H(G)_v\right)$, where $V(G)$ are the nodes of $G$, $H$ is a $d$-layer GNN with $H(G)_v$ is the output of $H$ on node $v$, and $M$ is an MLP. Note that we used the summation readout function which sums the output of the GNN over all the nodes of the input graph. 
    
    We will construct a new GNN $F'$ in the following way: Let $P$ be the set of $d$-patterns which appear in $P_1^d$ and $\tilde{P}$ all the $d$-patterns which appear in $P_2^d$ but not in $P_1^d$. Note that by the assumption that the distributions are with finite support, both $P$ and $\tilde{P}$ are finite. Suppose that $H(G)_v\in\mathbb{R}^k$, that is the dimension of the output feature vector for each $v\in V(G)$ is $k$-dimensional. We define a $d+2$-layer GNN $H'$ with output dimension $k+1$, such that:
    \begin{enumerate}
        \item For each $p\in P$ and node $v$ with $d$-pattern $p$, the output of $H'$ on $v$ is equal to $\begin{pmatrix}H(G)_v \\ 0 \end{pmatrix}$.
        \item For each $\tilde{p}\in \tilde{P}$ and node $\tilde{v}$ with $d$-pattern $\tilde{p}$, the output of $H'$ on $\tilde{v}$ is equal to $\begin{pmatrix}0 \\ 1 \end{pmatrix}$.
    \end{enumerate}
    This construction is possible using \thmref{thm:overfit} since both $P$ and $\tilde{P}$ are finite. We define an MLP $M'$ which have one more layer than $M$ in the following way: All the weights in all the layers except the last one are block matrices of the form $\begin{pmatrix} W ~ &\pmb{0} \\ \pmb{0}^\top ~ &1\end{pmatrix}$, where $W$ is the original weight matrix from $M$, and $\pmb{0}$ is a vector of zeroes of the corresponding size. Let $y\in\mathbb{R}$ be the maximum output of the task (in absolute value) over all the graphs from both $P_1$ and $P_2$. We define the last layer of $M'$ as: $\begin{pmatrix} 1 &0 \\0 & 2y \end{pmatrix}$. Finally, define $F'(G) = M'\left(\sum_{v\in V(G)}H'(G)_v\right)$. 
    
    We will now show the correctness of the construction. For any $d$-pattern $p\in P$, and $v$ with $d$-pattern $p$ it is clear that $H'(G)_v = \begin{pmatrix}H(G)_v \\ 0\end{pmatrix}$. Hence, for every graph $G$ coming from the distribution $P_1$ we have that 
    \[
    F'(G) = M'\left( \begin{pmatrix}\sum_{v\in V(G)}H(G)_v \\ 0\end{pmatrix} \right) =F(G)\] On the other side, let $G$ be some graph with from the distribution $P_2$. Then, there is some $\tilde{v}\in V(g)$ with $\tilde{v}\in\tilde{P}$, hence we have that $\sum_{v\in V(G)}H'(G)_v = \begin{pmatrix} \pmb{x}\\ z\end{pmatrix}$, where $\pmb{x}$ is some vector and $z\geq 1$. Hence, we have that $F'(G) = M'\left(\begin{pmatrix} \pmb{x}\\ z\end{pmatrix}\right) > y$. This means that the output of $F'$ on any graph drawn from $P_2$ is not the correct output for the task.
\end{proof}

\begin{proof}[Proof of \thmref{thm:overfit size node}]
 By the assumption, the output of the task is determined by the $d$-patterns of the nodes. For each node with pattern $p_i$ let $y_i$ be the output of the node prediction task. Define
\begin{equation}\label{eq:d_col_append}
    A=\arg\max_{A':P^{d}_1(A')<\epsilon}P^{d}_2(A')
\end{equation}
By \thmref{thm:overfit} there exists a first order GNN such that for any $d$-pattern $p_i\in A$ gives a wrong label and for every pattern $p_j\in \left(P_1 \cup P_2\right) \setminus A$ gives the correct label. Note that we can use \thmref{thm:overfit} since both $A$ and $\left(P_1 \cup P_2\right) \setminus A$ are finite, because we assumed that distributions on the graphs have finite supports. The 0-1 loss for small and large graphs is exactly $P^{d}_1(A)$ and $ P^{d}_2(A)$ respectively.

\end{proof}

\section{Additional experiments from \secref{sec: size gen problem empirical validation}}\label{append:addtional experiments size gen problem}

\subsection{Experiments on Larger Graphs}
\revision{We conducted the experiment from \figref{fig:validation_plots} (a) with much larger graphs. We used a three-layer GNN and tested on graphs with $n\in[50,500]$ nodes. See \figref{fig:up_to_500} The problem of size generalization persists, where increasing the graph size also significantly increases the loss on the test set. We stress that in all the experiments, the loss on the validation set is effectively zero. }

\subsection{Max-clique Size}
We consider the max-clique problem. The goal of this problem is given a graph, to output the size of the maximal clique. This is in general an NP-hard problem, hence a constant depth GNN will not be able to solve it for graphs of all sizes. For this task we sampled both the train and test graphs from a geometrical distribution, which resembles how graphs are created from point clouds, defined as follows: given a number of nodes $n$ and radius $\rho$ we draw $n$ points uniformly in $[0,1]^2$, each point correspond to a node in the graph, and two nodes are connected if their corresponding points have a distance less than $\rho$. We further analyzed the effects of how the network depth and architecture affect size generalization.

Table \ref{tab:max-clique} presents the test loss on the max-clique problem. Deeper networks are substantially more stricken due to the domain shift. If the test domain has a similar pattern distribution, increasing
the neural network depth from one layer to three layers results in a small decrement of at most $25\%$ to the loss. However, if the pattern distribution is different than the train pattern distribution, such change may increase the loss by more than $5.5\times$. We also show that the problem is consistent in both first-order GNN and GIN architectures.

\subsection{Different Architectures and Node/Graph Level Tasks}

We tested on the following tasks: (1) student-teacher task, on both graph and node levels, (2) per node degree prediction task, and (3) predicting the number of edges in the graph. The goal of these experiments is to show that the size-generalization problem persists on different tasks, different architectures, and its intensity is increased for deeper GNN. In all the experiments we draw the graphs from $G(n,p)$ distribution, wherein the test set $n$ is drawn uniformly between $40$ and $50$, and $p=0.3$, and in the test set $n=100$ and $p$ is either $0.3$ or $0.15$. We note that when $p=0.15$, the average degree of the test graph is equal to (approximately) the average degree of the train graph, while when $p=0.3$ it is twice as large. We would expect that the model will generalize better when the average degree in the train and test set is similar because then their  $d$-patterns will also be more similar.

\begin{wrapfigure}[21]{r}{0.25\textwidth}
\centering
\includegraphics[width=0.25\textwidth]{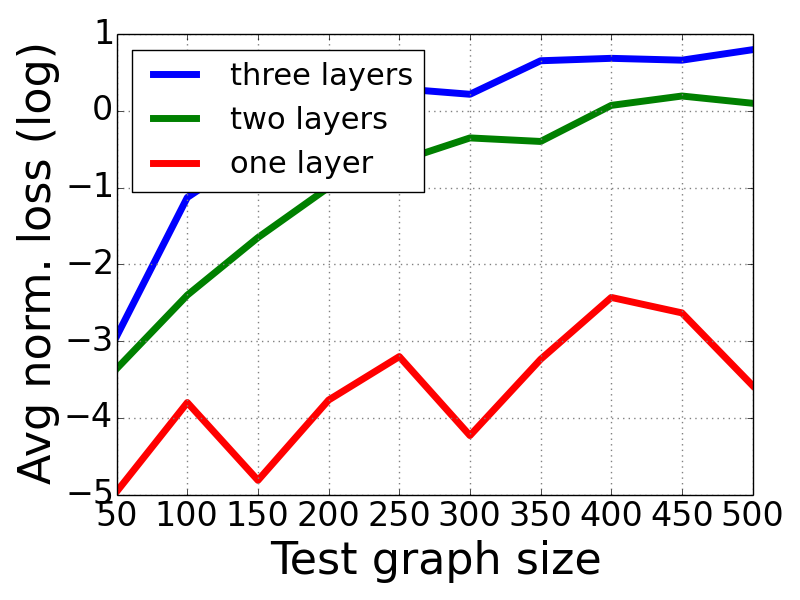}
\caption{Size generalization in PA graphs. The networks were trained on the edge count task, with graphs sampled from a preferential attachment model with $n$ drawn uniformly from $[10,50]$ and $m=4$. Tested on graphs also sampled the preferential attachment model with $n$ varies (x-axis) and $m=4$. For depth 2 and 3 GNN, as the graphs in the test set gets larger, the generalization worsens.} 
\label{fig:PA_exp}
\end{wrapfigure}
Table \ref{tab:student-teacher} compares the performance of these tasks when changing the graph size of the test data. We tested the performance with the squared loss. We note that the GNNs (both first-order GNN and GIN) successfully learned all the tasks presented here on the train distribution, here we present their generalization capabilities to larger sizes.

\subsection{Different data distribution - Preferential Attachment}

We performed an experiment on the preferential attachment graph model. In this model, the graph distribution is determined by two parameters $n$ and $m$. Each graph is created sequentially, where at each iteration a new node is added to the graph, up to $n$ nodes. Each new node is connected to exactly $m$ existing nodes, where the probability to connect to node $v$ is proportional to its degree. This means that higher degree nodes have a higher degree to have more nodes connected to them. 

We trained on the task of edge count (i.e. predicting the number of edges in a graph) for graphs sampled from a preferential attachment model with $n$ uniformly sampled from $[10,50]$ and $m=4$. We tested on graphs also sampled from a preferential attachment model with $n$ nodes for $n$ varying from $50$ to $500$, and $m=4$. Note that this task can be solved efficiently for any graph distribution. \figref{fig:PA_exp} depicts the results. It is clear that for depth 2 and 3 GNNs, as the graph size gets larger, the GNN fails to predict the number of edges. This shows that although for this problem there is a solution that works on all graph sizes, and we trained on graphs of varying sizes, GNNs tend to converge to solutions that do not generalize to larger sizes. 

\subsection{Generalization From Large to Small Graphs}
\revision{We additionally tested size generalization in the opposite direction, i.e. generalizing from large to small graphs. We used the same teacher-student setting as in \figref{fig:validation_plots}, where the graphs are drawn from a $G(n,p)$ distribution with a constant $p=0.3$. We consider three separate training sets with graphs of sizes $[90,100],~[140,150],~[190,200]$ and tested on graphs of sizes $50,~75$. The rows of  Table \ref{tab:large to small} clearly show that when the size difference between the graphs in the train and test sets decreases, the loss also decreases, which is consistent with our theory.}

\begin{figure}
    \centering
    \includegraphics[width=0.45\textwidth]{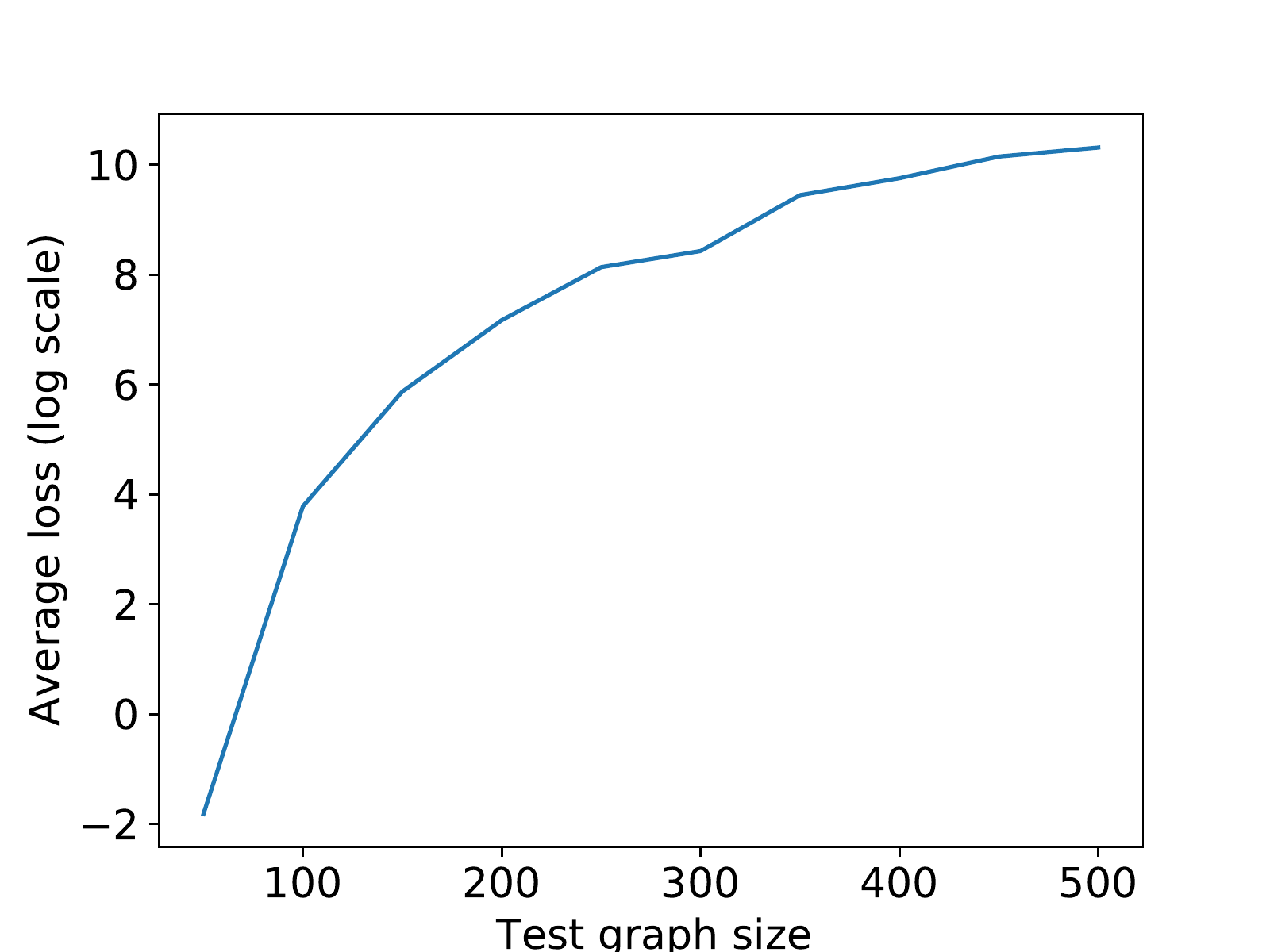}
    \caption{\revision{Extending the experiment in \figref{fig:validation_plots} (a) to larger graphs. We train on graphs with bounded size $n\in [40,50]$ and test on graphs with size up to $n=500$ with constant $p=0.3$.}}
    \label{fig:up_to_500}
\end{figure}

\begin{table*}[t]
\centering
    \footnotesize

\begin{tabular}{|c|c|c|c|c|c|c|}
\hline 
 & \multicolumn{3}{c|}{first order GNN} & \multicolumn{3}{c|}{GIN}\\
\hline 
 $\rho_{train}/\rho_{test}$ & 1-layer & 2-layers & 3-layers & 1-layer & 2-layers & 3-layers\\
\hline 
\hline 
$1$ & $402 \pm 59$ & $926 \pm 245$ & $2325 \pm 3613$ & $367 \pm 56$ & $620 \pm 130$ & $634 \pm 1042$ \\
\hline 
$\sqrt{2}$ & $96 \pm 5$ & $111 \pm 4$ & $119 \pm 5$ & $101 \pm 3$ & $114 \pm 6$ & $123 \pm 14$\\
\hline 
\end{tabular}

\caption{\label{tab:max-clique} The difference is the predicted max clique size under size generalization domain shift. The train domain graphs were constructed by drawing $n\in[40,50]$ points uniformly in the unit square, and connecting two points if their distance is less than $\rho_{train}=0.3$. The test set domain graphs contain $n=100$ nodes, effectively increasing their density by $2$. We tested with two different values of $\rho_{train}/{\rho_{test}}$, the ratio between the train and test connectivity radius. A proper scaling that keeps the expected degree of each node is $\rho = \sqrt{2}$. Here, although proper scaling does not help solve the problem completely, it does improve performance. }
\end{table*}

\begin{table*}[t]
    \footnotesize
    \centering{}

        \begin{tabular}{|c|c|c|c|c|c|}
    \hline 
     &  & \multicolumn{2}{c|}{Node regression} & \multicolumn{2}{c|}{Graph regression}\\
    \hline 
     &  $p$ & Student - Teacher  & Degree & Student - Teacher & Edge count\\
    \hline 
    \hline 
    \multirow{2}{*}{first order GNN \cite{morris2019weisfeiler}} & $0.3$ & $3500 \pm 9240$ & $348 \pm 553$ & $(1.8 \pm 3)\cdot 10^4$ & $(1.8 \pm 2.4)\cdot 10^6$\\
    \cline{2-6} \cline{3-6} \cline{4-6} \cline{5-6} \cline{6-6} 
     & $0.15$ & $\boldsymbol{0.02 \pm 0.05}$ & $\boldsymbol{\left(1.2\pm0.8\right)\cdot10^{-3}}$ & $\boldsymbol{0.04 \pm 0.04}$ & $\boldsymbol{5.1 \pm 2.5}$\\
    \hline 
    \multirow{2}{*}{GIN} & $0.3$ & $1384 \pm 3529$ & $73 \pm 86$ & $5487 \pm 9417$ & $(4.2 \pm 4.3)\cdot 10^5$\\
    \cline{2-6} \cline{3-6} \cline{4-6} \cline{5-6} \cline{6-6} 
     & $0.15$ & $\mathbf{0.96 \pm 0.98}$ & $\mathbf{0.33 \pm 0.04}$ & $\mathbf{0.04 \pm 0.07}$ & $\mathbf{6.7 \pm 5.4}$\\
    \hline 
    \end{tabular}
    \caption{\label{tab:student-teacher} Comparing performance on different local distributions (a) A student-teacher graph regression task; (b) A graph regression task, where the graph label is the number of edges; (c) A student-teacher node regression task; (d) A node regression task, where the node label is its degree. In the edge count/degree tasks the loss is the mean difference from the ground-truths, divided by the average degree/number of edges. In the student-teacher tasks the loss is the mean $L_2$ loss between the teacher's value and the student's prediction, divided by the averaged student's prediction. Both the student and teacher share the same 3-layer architecture}
\end{table*}

\begin{table*}[t]
\centering
    \footnotesize
\begin{tabular}{|c|c|ccc|}
\hline
\multicolumn{2}{|c|}{\multirow{2}{*}{}}&\multicolumn{3}{|c|}{\textbf{Train graph size}} \\\hline
\multicolumn{2}{|c|}{}&$\mathbf{[90,100]}$&$\mathbf{[140,150]}$&$\mathbf{[190,200]}$ \\
    \hline
        \multirow{3}{*}{\textbf{Test  graph  size}}&$\mathbf{50}$&1.87&3.1&4.36\\
        &$\mathbf{75}$&1.93&4.19&4.29\\
    \hline
\end{tabular}

\caption{\label{tab:large to small} \revision{Training on large graphs and testing on smaller graphs. Teacher-student task with graphs drawn from a $G(n,p)$ distribution with $p=0.3$ and $n$ varies, as described in the table. Results are on a logarithmic scale. As expected, even generalizing from large to small graphs is difficult, but the results improve as the sizes of the graphs in the train set are close to the sizes of the graphs in the test set.}}
\end{table*}





\section{SSL task on $d$-pattern  tree}\label{appen:SSL_tree_task}
First, we will need the following definition which constructs a tree out of the $d$-pattern introduced in \secref{sec:local graph patterns}. This tree enables us to extract certain properties for each node which can, later on, be learned using a GNN. This definition is similar to the definition of "unrolled tree" from \cite{morris2019towards}.

\begin{definition}[$d$-pattern tree]
Let $G=(V,E)$ a graph, $C$ a finite set of node features, where each $v\in V$ have a corresponding feature $c_v$, and $d \geq 0$. For a node $v\in V$, its \textbf{$d$-pattern tree} $ T^{(d)}_v=(V^{(d)}_v,E^{(d)}_v)$ is directed tree where each node corresponds to some node in G. It is defined recursively in the following way: For $d=0$, $V^{(0)}_v = u_{(0,v)}$, and $E^{(0)}_v= \varnothing$. Suppose we defined $T^{(d-1)}_v$, and let $\tilde{V}_v^{(d-1)}$ be all the leaf nodes in $V^{(d-1)}_v$ (i.e. nodes without incoming edges). We define:

\begin{align*}
    V^{(d)}_v &= V^{(d-1)}_v \cup \left\{u_{(d,v')}: v'\in \mathcal{N}(v''),~ u'_{(d-1,v'')}\in\tilde{V}^{(d-1)}_v\right\} \\
    E^{(d)}_v &= E^{(d-1)}_v \cup \left\{(u_{(d,v')},u'_{(d-1,v'')}): v'\in \mathcal{N}(v''),~ u'_{(d-1,v'')}\in\tilde{V}^{(d-1)}_v\right\}
\end{align*}
and for every node $u_{(d,v')}\in V_v^{(d)}$, its node feature is $c_{v'}$ - the node feature of $v'$

\end{definition}

The main advantage of pattern trees is that they encode all the information that a GNN can produce for a given node by running the same GNN on the pattern tree. 

This tree corresponds to the local patterns in the following way: the $d$-pattern tree of a node can be seen as a multiset of the children of the root, and each child is a multiset of its children, etc. This completely describes the $d$-pattern of a node. In other words, there is a one-to-one correspondence between $d$-patterns and pattern trees of depth $d$. Thus, a GNN that successfully represents the pattern trees of the target distribution will also successfully represent the $d$-patterns of the target distribution.

Using the $d$-pattern tree we construct a simple regression SSL task where its goal is for each node to count the number of nodes in each layer and of each feature of its $d$-pattern tree. This is a simple descriptor of the tree, which although loses some information about connectivity, does hold information about the structure of the layers. 

For example, in \figref{fig:pattern tree} the descriptor for the tree would be that the root (zero) layer has a single black node, the first layer has two yellow nodes, the second layer has two yellow, two gray, and two black nodes, and the third layer has ten yellow, two black and two gray nodes.




\section{Training Procedure}\label{appen:training procedure}
In this section, we explain in detail the training procedure used in the experiments of \secref{sec:improve size gen}. Let $X_{Main}$ and $X_{SSL}$ be two labeled datasets, the first contains the labeled examples for the main task from the source distribution, and the second contains examples labeled by the SSL task from both the source and target distributions. Let $\ell:\reals\times\reals\rightarrow\reals$ be a loss function, in all our experiments we use the cross-entropy loss for classification tasks and squared loss for regression tasks. We construct the following models:

(1) $f_{GNN}$ is a GNN feature extractor. Its input is a graph and its output is a feature vector for each node in the graph. (2) $h_{Main}$ is a head (a small neural network) for the main task. Its inputs are the node feature and it outputs a prediction (for graph prediction tasks this head contains a global pooling layer). (3) $h_{SSL}$ is the head for the SSL task. Its inputs are node features, and it outputs a prediction for each node of the graph, depending on the specific SSL task used.

\textbf{Pretraining.} Here, there are two phases for the learning procedure. In the first phase, at each iteration we sample a batch $(\bx_1,\by_1)$ from $X_{SSL}$, and train by minimizing the objective: $\ell(h_{SSL}\circ f_{GNN} (\bx_1), \by_1)$. In this phase both $h_{SSL}$ and $f_{GNN}$ are trained. In the second phase, at each iteration we sample $(\bx_2,\by_2)$ from $X_{main}$ and train on the loss  $\ell(h_{Main}\circ f_{GNN} (\bx_2), \by_2)$, where we only train the head $h_{Main}$, while the weights of $f_{GNN}$ are fixed.

\textbf{Multitask training.} Here we train all the functions at the same time. At each iteration we sample a batch $(\bx_1,\by_1)$ from $X_{SSL}$ and a batch $(\bx_2,\by_2)$ from $X_{Main}$ and train by minimizing the objective:
\[
\alpha\cdot\ell(h_{SSL}\circ f_{GNN} (\bx_1), \by_1) + (1-\alpha)\cdot\ell(h_{Main}\circ f_{GNN} (\bx_2), \by_2).
\] 
Here $\alpha\in [0,1]$ is the weight for the SSL task, in all our experiments we used $\alpha = 1/2$.

For an illustration of the training procedures see \figref{fig:SSL training}. These procedures are common practices for training with SSL tasks (see e.g. \cite{you2020does}).

We additionally use a semi-supervised setup in which we are given a dataset $X_{FS}$ of few-shot examples from the target distribution with their correct label. In both training procedures, at each iteration we sample a batch $(\bx_3,\by_3)$ from $X_{FS}$ and add a loss term $\beta\ell(h_{Main}\circ f_{GNN}(\bx_3),\by)$ where $\beta\in [0,1]$ is the weight of the few-shot loss. In pretraining, this term is only added to the second phase, with weight $1/2$ and adjust the weight of the main task to $1/2$ as well (equal weight to the main task). In the multitask setup, we add this term with weight $1/3$ and adjust the weights of the two other losses to $1/3$ as well, so all the losses have the same weight.
\begin{figure}
    \centering
    {\includegraphics[width=0.45\textwidth]{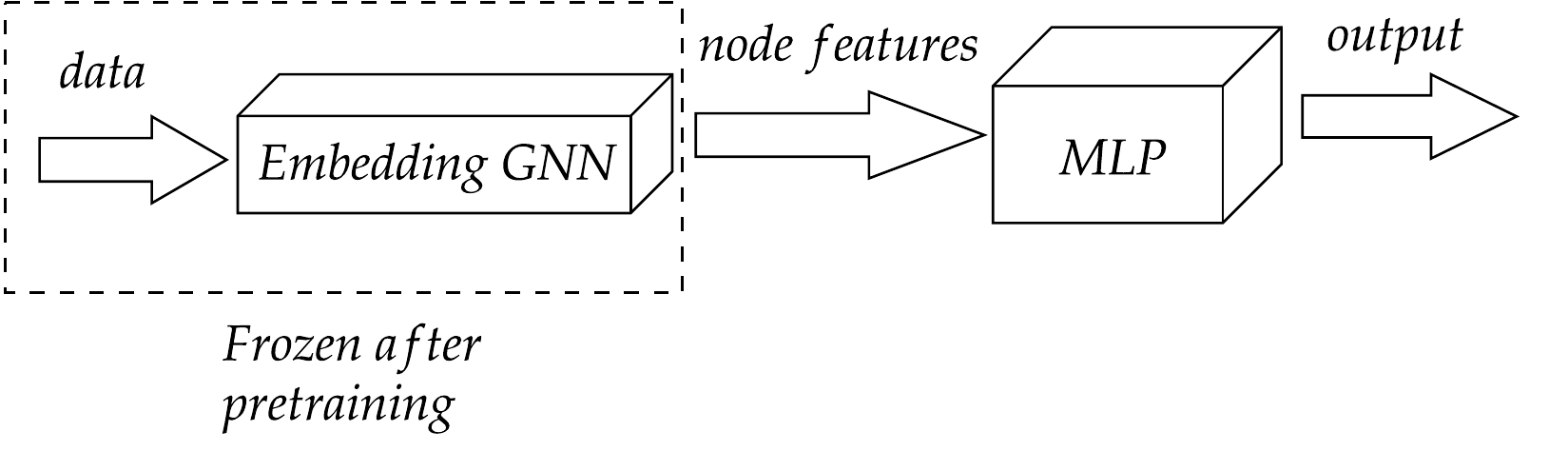}}
    {\includegraphics[width=0.45\textwidth]{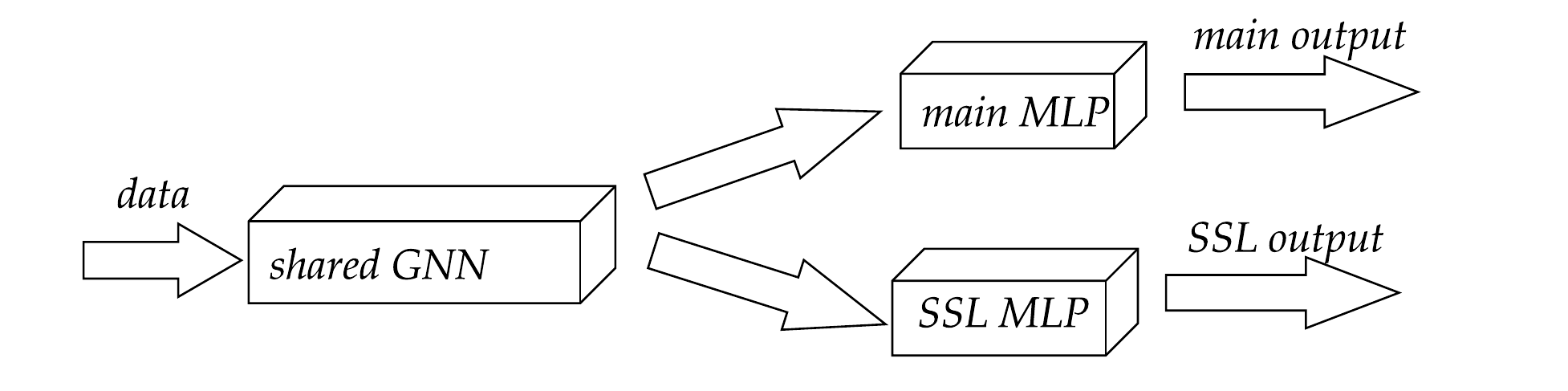}}
    \caption{Two training procedures for learning with SSL tasks. \textbf{Left:} Learning with pretraining: Here, a GNN is trained on the SSL task with a specific SSL head. After training, the weights of the GNN are fixed, and only the main head is trained on the main task. \textbf{Right:} Multitask learning: Here, there is a shared GNN and two separate heads, one for the SSL task and one for the main task. The GNN and both heads are trained simultaneously.}
    \label{fig:SSL training}
\end{figure}

\section{More experiments from \secref{sec:improve size gen}}\label{sec:more experiments from sec solution}
\subsection{Synthetic datasets}
We used the setting of Section \ref{sec: size gen problem empirical validation}. Source graphs were generated with $G(n,p)$ with $n$ sampled uniformly in $[40, 50]$ and $p=0.3$. Target graphs were sampled from $G(n,p)$ with $n=100$ and $p=0.3$. 

Table \ref{tab:synthetic datasets} depicts the results of using the  $d$-patterns SSL tasks, in addition to using the semi-supervised setting. It can be seen that adding the  $d$-patterns SSL task significantly improves the performance on the teacher-student task, although it does not completely solve it. We also observe that adding labeled examples from the target domain significantly improves the performance of all tasks. Note that adding even a single example significantly improves performance. In all the experiments, the network was successful at learning the task on the source domain with less than $0.15$ averaged error, and in most cases much less. 

\figref{fig:embedding-vanilla-vary-p} depicts a side-by-side plot of the 3-layer case of \figref{fig:validation_plots} (d), where training is done on graphs sampled from $G(n,p)$ with $40$ to $50$ nodes and $p=0.3$, and testing on graphs with $100$ nodes and $p$ varies. We compare Vanilla training, and our pattern tree SSL task with pretraining. It is clear that for all values of $p$ our SSL task improves over vanilla training, except for $p=0.15$.

\begin{figure}[t]
    \centering
    \includegraphics[width=0.45\textwidth]{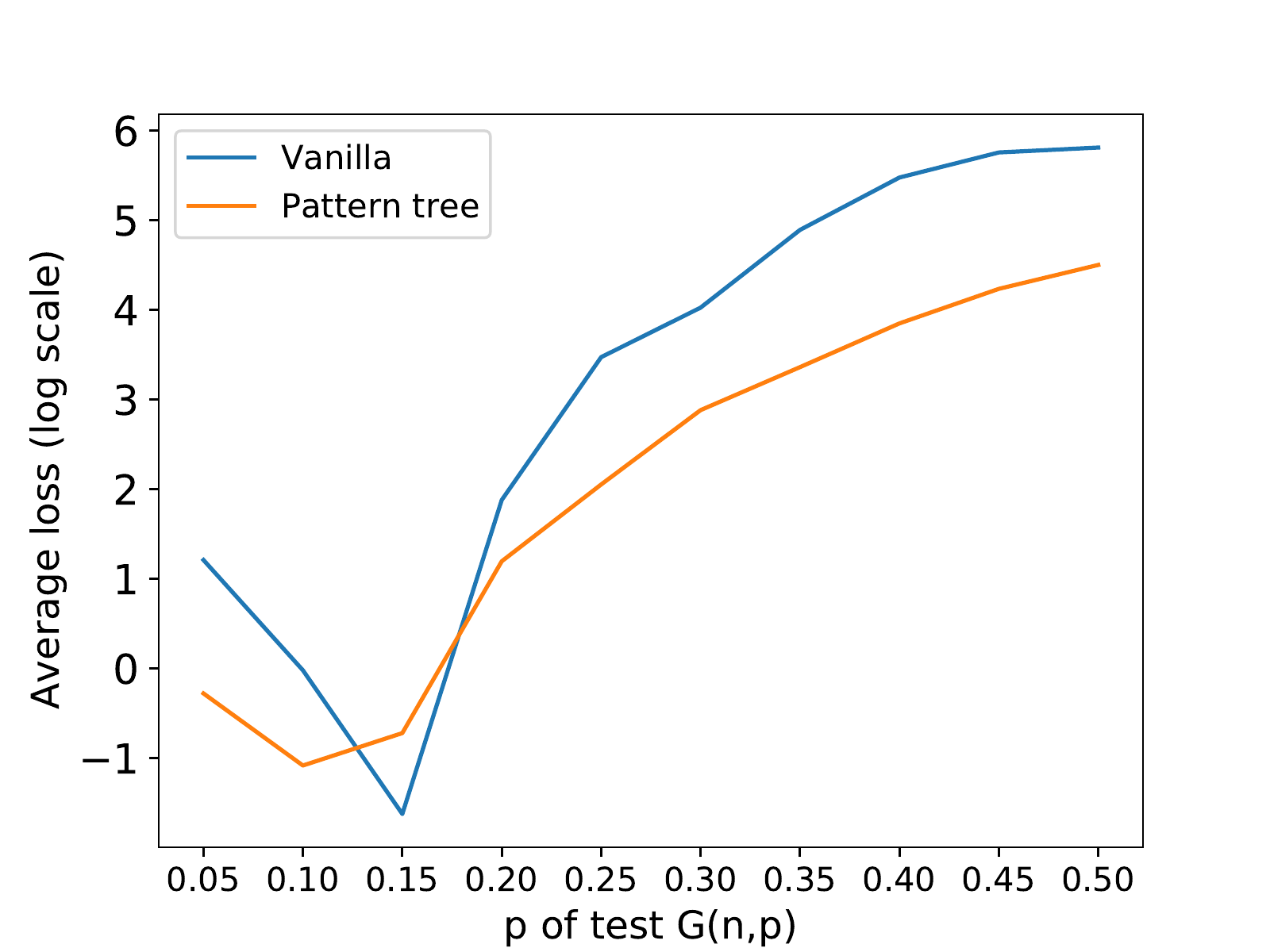}
    \caption{Teacher-student setup with a 3-layer GNN. Training is on graphs drawn i.i.d from $G(n,p)$ with $n\in\{40,...,50\}$ uniformly and $p=0.3$. Testing is done on graphs with $n=100$ and $p$ vary (x-axis). The "Pattern tree" plot represent training with our pattern tree SSL task, using the pretraining setup.}
    \label{fig:embedding-vanilla-vary-p}
\end{figure}

\begin{table*}[t]
    \centering
        \footnotesize
\begin{sc}
    
        \begin{tabular}{|l|l|l|l|l|l|}
    \hline
    \textbf{Tasks}    & \textbf{\# target} & \textbf{Edge} & \textbf{Degree}   & \textbf{Teacher} & \textbf{Teacher student} \\ 
    \textbf{}    & \textbf{samples} & \textbf{count} & \textbf{}   & \textbf{Student} & \textbf{Per node} \\ \hline \hline
\multirow{4}{*}{\textbf{Vanilla}}    & \textbf{0}    & $(1.9 \pm 2.1)\cdot 10^6$     & $363 \pm 476$ & $(3 \pm 5)\cdot 10^4$        
          & $3311 \pm 8813$                 \\ 
          & \textbf{1}    & $679 \pm 1014$     & $0.27 \pm 0.27$   & $53 \pm 98$ & $1.1 \pm 2.8$                   \\ 
          & \textbf{5}    & $95 \pm 105$     &    $ (3.8 \pm 6.1)\cdot 10^{-2}$ & $3.1 \pm 5.4$ & $0.4 \pm 1.3$                   \\ 
          & \textbf{10}   & $43 \pm 39$     & $(1.8 \pm 3.5)\cdot 10^{-2} $   & $25 \pm 75$ & $(0.8 \pm 1.3)\cdot 10^{-1}$                   \\ \hline
\multirow{4}{*}{\textbf{$d$-Pattern PT}} & \textbf{0}    & $(1.9 \pm 1.4)\cdot 10^6$      & $1580 \pm 1912$ 
          & $809 \pm 1360$ & $4.4 \pm 5.5$                    \\ 
          & \textbf{1}    & $2528 \pm 1559$      & $2 \pm 1.2$   & $0.5 \pm 1.5 $ & $(4 \pm 6)\cdot 10^{-3} $                   \\ 
          & \textbf{5}    & $134 \pm 279$     & $0.55 \pm 0.14$  & $0.11 \pm 0.19 $ & $(2 \pm 3)\cdot 10^{-3}$                   \\ 
          & \textbf{10}   & $79 \pm 79$     & $0.55 \pm 0.1$   & $1.4 \pm 4.4$ & $(2\pm 3)\cdot 10^{-3} $                   \\ \hline
    \end{tabular}
    
    \end{sc}
    \caption{Results on synthetic datasets (Vanilla vs. $d$-pattern PT). }
    \label{tab:synthetic datasets}
\end{table*}

\subsection{Max clique}
\revision{We tested our SSL task on the task of finding the max-clique size of a given graph, similar to the experiment presented in Table \ref{tab:max-clique}. For the train distribution, we constructed graphs by drawing $n\in[10,50]$ points uniformly in the unit square and connecting two points if their distance is less than $\rho=0.3$. The test set graphs are drawn similarly, but with $n=100$, effectively increasing the density by 2. We trained on $20,000$ graphs and tested on $2000$ graphs, where the task is to predict the size of the maximal clique in the graph. We used squared loss. We ran the experiment 10 times and averaged the loss over all the experiments.}

\revision{Without our SSL, the squared loss on average is 2325, hence not generalizing well to larger sizes. With our SSL task and using pretraining (PT) the loss on average is $1327$, improving over vanilla training, but still not solving the problem. Using our SSL task with multitask training (MTL) the average loss provided worse results than vanilla training. We note that this is in general an NP-hard problem, hence solving it might require a specific solution, while our SSL is a general framework for improving size generalization.}

\subsection{OGB dataset}
\revision{We additionally tested our framework on 8 tasks from the "ogbg-molpcba" dataset from the OGB dataset collection \cite{hu2020open}. These are binary classification tasks, where we choose tasks for which there is no one class with more than $90\%$ of the samples (i.e. the tasks with the most balanced labels). We compare our SSL task with pretraining vs. vanilla training. We note that since our task assumes only categorical features and bidirectional edges with no features, we did not use the node and edge features. We split the datasets by size, in the same manner, we did in the experiments from Table \ref{tab:real datasets}, where we trained on the $50\%$ smallest graphs and tested on the $10\%$ largest graphs. We report the accuracy. For the results see Table \ref{tab:ogb}. The results are inconclusive due to the lack of features which is valuable information not used by our model. We also tested using the given split from the OGB repository but found the same inconclusive results. In the future, it will be interesting to generalize our method to handle continuous node features and edge features. We hope that by using these features our method could also improve on the size generalization task for datasets from the OGB collection.}

\begin{table*}[t]
    \scriptsize

    \centering{}

 \begin{tabular}{|l|l|l|l|l|l|l|l|l|l|}
\hline
\textbf{Task No.} & \textbf{0} & \textbf{47} & \textbf{93} & \textbf{94} & \textbf{110} & \textbf{111} & \textbf{114} & \textbf{115} & \textbf{average}\\ \hline
\textbf{Pattern PT} & $88.5 \pm 0$ & $78.4 \pm 25$     &  $55.9 \pm 9.2$ & $68.4 \pm 0$   & $73.1 \pm 0$ &     $44.8 \pm 0.2$ &    $84.9 \pm 7.3$ & $73.7 \pm 14.9$   &   $70.9\%$   \\ \hline
\textbf{Vanilla}     &  $88.4\pm 0.4$ & $ 82.1 \pm 23.7$     & $59.1 \pm 5.6$  & $68.4 \pm 0.1$  &  $70.3 \pm 6.1$   &  $44.9 \pm 0$   & $85.78 \pm 4.8$     &  $76.5 \pm 5.7$   & $71.9\%$ \\ \hline
\end{tabular}
    \caption{\label{tab:ogb} \revision{Results on the ogbg-molpcba datasets from the OGB collection. The dataset contains 128 different tasks, we used the tasks with the most balanced labels (there is no class with more than $90\%$ of the samples). In both settings, we emitted the node and edge features as our method does not support edge features and continuous node features (only categorical node features). The results are inconclusive, with a slight edge toward vanilla training.}}
\end{table*}

\subsection{Datasets statistics}\label{appen:datasets statistics}
Table \ref{tab:dataset statistics} shows the statistics of the datasets that were used in the paper. In particular, the table presents the split that was used in the experiments, we trained on graphs with sizes smaller or equal to the 50-th percentile and tested on graphs with sizes larger or equal to the 90-th percentile. Note that all the prediction tasks for these datasets are binary classification. In all the datasets there is a significant difference between the graph sizes in the train and test sets, and in some datasets, there is also a difference between the distribution of the output class in the small and large graphs.
\begin{table*}[t]\label{tab:dataset statistics}
\caption{Dataset statistics.}
\centering
\tiny
\begin{tabular}{llllllllll}
\cline{2-10}
\multicolumn{1}{l|}{}                         & \multicolumn{3}{c|}{\textbf{NCI1}}                                                                                           & \multicolumn{3}{c|}{\textbf{NCI109}}                                                                                         & \multicolumn{3}{c|}{\textbf{DD}}                                                                                             \\ \cline{2-10} 
\multicolumn{1}{l|}{}                         & \multicolumn{1}{l|}{\textbf{all}} & \multicolumn{1}{l|}{\textbf{Smallest 50\%}} & \multicolumn{1}{l|}{\textbf{Largest 10\%}} & \multicolumn{1}{l|}{\textbf{all}} & \multicolumn{1}{l|}{\textbf{Smallest 50\%}} & \multicolumn{1}{l|}{\textbf{Largest 10\%}} & \multicolumn{1}{l|}{\textbf{all}} & \multicolumn{1}{l|}{\textbf{Smallest 50\%}} & \multicolumn{1}{l|}{\textbf{Largest 10\%}} \\ \hline
\multicolumn{1}{|l|}{\textbf{Class A}}        & \multicolumn{1}{l|}{49.95\%}      & \multicolumn{1}{l|}{62.30\%}                & \multicolumn{1}{l|}{19.17\%}               & \multicolumn{1}{l|}{49.62\%}      & \multicolumn{1}{l|}{62.04\%}                & \multicolumn{1}{l|}{21.37\%}               & \multicolumn{1}{l|}{58.65\%}      & \multicolumn{1}{l|}{35.47\%}                & \multicolumn{1}{l|}{79.66\%}               \\ \hline
\multicolumn{1}{|l|}{\textbf{Class B}}        & \multicolumn{1}{l|}{50.04\%}      & \multicolumn{1}{l|}{37.69\%}                & \multicolumn{1}{l|}{80.82\%}               & \multicolumn{1}{l|}{50.37\%}      & \multicolumn{1}{l|}{37.95\%}                & \multicolumn{1}{l|}{78.62\%}               & \multicolumn{1}{l|}{41.34\%}      & \multicolumn{1}{l|}{64.52\%}                & \multicolumn{1}{l|}{20.33\%}               \\ \hline
\multicolumn{1}{|l|}{\textbf{Num of graphs}}           & \multicolumn{1}{l|}{4110}         & \multicolumn{1}{l|}{2157}                   & \multicolumn{1}{l|}{412}                   & \multicolumn{1}{l|}{4127}         & \multicolumn{1}{l|}{2079}                   & \multicolumn{1}{l|}{421}                   & \multicolumn{1}{l|}{1178}         & \multicolumn{1}{l|}{592}                    & \multicolumn{1}{l|}{118}                   \\ \hline
\multicolumn{1}{|l|}{\textbf{Avg graph size}} & \multicolumn{1}{l|}{29}           & \multicolumn{1}{l|}{20}                     & \multicolumn{1}{l|}{61}                    & \multicolumn{1}{l|}{29}           & \multicolumn{1}{l|}{20}                     & \multicolumn{1}{l|}{61}                    & \multicolumn{1}{l|}{284}          & \multicolumn{1}{l|}{144}                    & \multicolumn{1}{l|}{746}                   \\ \hline
                                              &                                   &                                             &                                            &                                   &                                             &                                            &                                   &                                             &                                            \\ \cline{2-10} 
\multicolumn{1}{l|}{}                         & \multicolumn{3}{c|}{\textbf{Twitch\_egos}}                                                                                   & \multicolumn{3}{c|}{\textbf{Deezer\_egos}}                                                                                   & \multicolumn{3}{c|}{\textbf{IMDB-binary}}                                                                                    \\ \hline
\multicolumn{1}{|l|}{}                        & \multicolumn{1}{l|}{\textbf{all}} & \multicolumn{1}{l|}{\textbf{Smallest 50\%}} & \multicolumn{1}{l|}{\textbf{Largest 10\%}} & \multicolumn{1}{l|}{\textbf{all}} & \multicolumn{1}{l|}{\textbf{Smallest 50\%}} & \multicolumn{1}{l|}{\textbf{Largest 10\%}} & \multicolumn{1}{l|}{\textbf{all}} & \multicolumn{1}{l|}{\textbf{Smallest 50\%}} & \multicolumn{1}{l|}{\textbf{Largest 10\%}} \\ \hline
\multicolumn{1}{|l|}{\textbf{Class A}}        & \multicolumn{1}{l|}{46.23\%}      & \multicolumn{1}{l|}{39.05\%}                & \multicolumn{1}{l|}{58.07\%}               & \multicolumn{1}{l|}{56.80\%}      & \multicolumn{1}{l|}{44.78\%}                & \multicolumn{1}{l|}{64.97\%}               & \multicolumn{1}{l|}{50.00\%}      & \multicolumn{1}{l|}{48.98\%}                & \multicolumn{1}{l|}{55.55\%}               \\ \hline
\multicolumn{1}{|l|}{\textbf{Class B}}        & \multicolumn{1}{l|}{53.76\%}      & \multicolumn{1}{l|}{60.94\%}                & \multicolumn{1}{l|}{41.92\%}               & \multicolumn{1}{l|}{43.19\%}      & \multicolumn{1}{l|}{55.21\%}                & \multicolumn{1}{l|}{35.02\%}               & \multicolumn{1}{l|}{50.00\%}      & \multicolumn{1}{l|}{51.01\%}                & \multicolumn{1}{l|}{44.44\%}               \\ \hline
\multicolumn{1}{|l|}{\textbf{Num of graphs}}           & \multicolumn{1}{l|}{127094}       & \multicolumn{1}{l|}{65016}                  & \multicolumn{1}{l|}{14746}                 & \multicolumn{1}{l|}{9629}         & \multicolumn{1}{l|}{4894}                   & \multicolumn{1}{l|}{968}                   & \multicolumn{1}{l|}{1000}         & \multicolumn{1}{l|}{543}                    & \multicolumn{1}{l|}{108}                   \\ \hline
\multicolumn{1}{|l|}{\textbf{Avg graph size}} & \multicolumn{1}{l|}{29}           & \multicolumn{1}{l|}{20}                     & \multicolumn{1}{l|}{48}                    & \multicolumn{1}{l|}{23}           & \multicolumn{1}{l|}{13}                     & \multicolumn{1}{l|}{68}                    & \multicolumn{1}{l|}{19}           & \multicolumn{1}{l|}{13}                     & \multicolumn{1}{l|}{41}                    \\ \hline
                                              &                                   &                                             &                                            &                                   &                                             &                                            &                                   &                                             &                                            \\ \cline{2-4}
\multicolumn{1}{l|}{}                         & \multicolumn{3}{c|}{\textbf{PROTEINS}}                                                                                       &                                   &                                             &                                            &                                   &                                             &                                            \\ \cline{2-4}
\multicolumn{1}{l|}{}                         & \multicolumn{1}{l|}{\textbf{all}} & \multicolumn{1}{l|}{\textbf{Smallest 50\%}} & \multicolumn{1}{l|}{\textbf{Largest 10\%}} &                                   &                                             &                                            &                                   &                                             &                                            \\ \cline{1-4}
\multicolumn{1}{|l|}{\textbf{Class A}}        & \multicolumn{1}{l|}{59.56\%}      & \multicolumn{1}{l|}{41.97\%}                & \multicolumn{1}{l|}{90.17\%}               &                                   &                                             &                                            &                                   &                                             &                                            \\ \cline{1-4}
\multicolumn{1}{|l|}{\textbf{Class B}}        & \multicolumn{1}{l|}{40.43\%}      & \multicolumn{1}{l|}{58.02\%}                & \multicolumn{1}{l|}{9.82\%}                &                                   &                                             &                                            &                                   &                                             &                                            \\ \cline{1-4}
\multicolumn{1}{|l|}{\textbf{Num of graphs}}           & \multicolumn{1}{l|}{1113}         & \multicolumn{1}{l|}{567}                    & \multicolumn{1}{l|}{112}                   &                                   &                                             &                                            &                                   &                                             &                                            \\ \cline{1-4}
\multicolumn{1}{|l|}{\textbf{Avg graph size}} & \multicolumn{1}{l|}{39}           & \multicolumn{1}{l|}{15}                     & \multicolumn{1}{l|}{138}                   &                                   &                                             &                                            &                                   &                                             &                                            \\ \cline{1-4}
\end{tabular}
\end{table*}



\subsection{Correlation Between Size Discrepancy and $d$-pattern Discrepancy}\label{appen:counting patterns}

In this section, we show that in many real datasets, there is a high correlation between the sizes of the graphs and their $d$-patterns. This motivates the effect that $d$-patterns have on the ability of GNNs to generalize to larger sizes than they were trained on. 

We focused on the datasets that were tested in \secref{sec:improve size gen}: \textbf{Biological datasets:} NCI1, NCI109, D \& D, Proteins. \textbf{Social networks:} IMDB-Binary, Deezer ego nets, Twitch ego nets. To this end, we split each dataset to the $50\%$ smallest graphs and $10\%$ largest graphs. We calculated the distribution of $2$-patterns for the $20$ most common two patterns in each split (smallest and largest graphs) and compared these two distributions. The results are depicted in \figref{fig:2-pattern hist} (The plot for the Twitch dataset similar to the one of Deezer, where the distributions of $2$-patterns are disjoint).

It is clear that for the NCI1, NCI109, and D \& D datasets there is a very high overlap of $2$-patterns between small and large graphs. Indeed, in the result from Table \ref{tab:real datasets} it can be seen that adding an SSL task (and specifically our tree pattern task) does not improve significantly over vanilla training (except for NCI109). On the other hand, for the other datasets, there is a very clear discrepancy between the $2$-patterns of small and large graphs. Indeed for these datasets, our SSL task improved over vanilla training. For the social network datasets, it is even more severe, where there is almost no overlap between the $2$-patterns, meaning that the small and large graphs have very different local structures. We also calculated the total variation distance between the distributions for every dataset, this appears in the first row of Table \ref{tab:real datasets}.

\begin{figure}[ht]
    \centering
    {\includegraphics[width=0.45\textwidth]{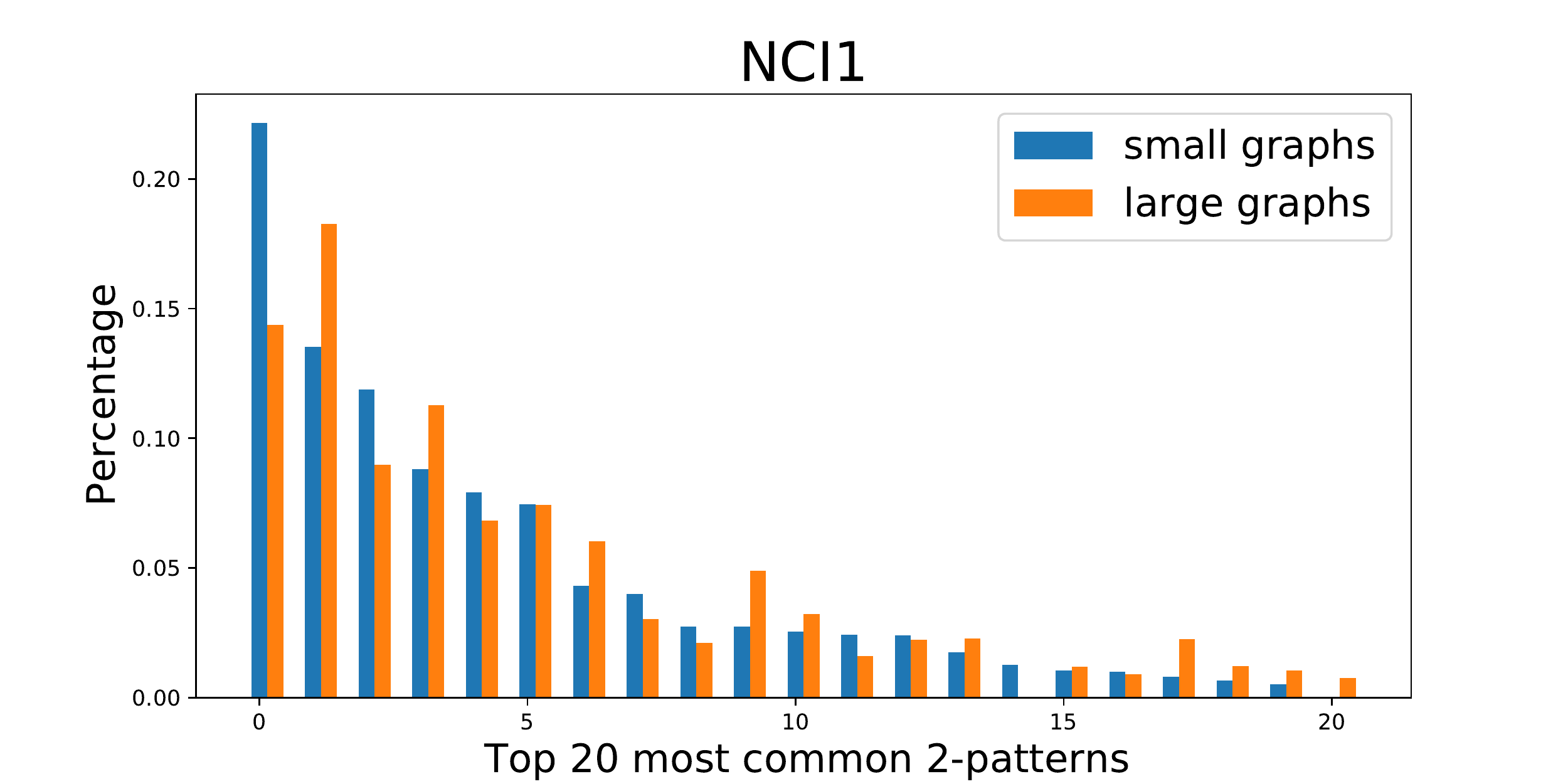}}
    {\includegraphics[width=0.45\textwidth]{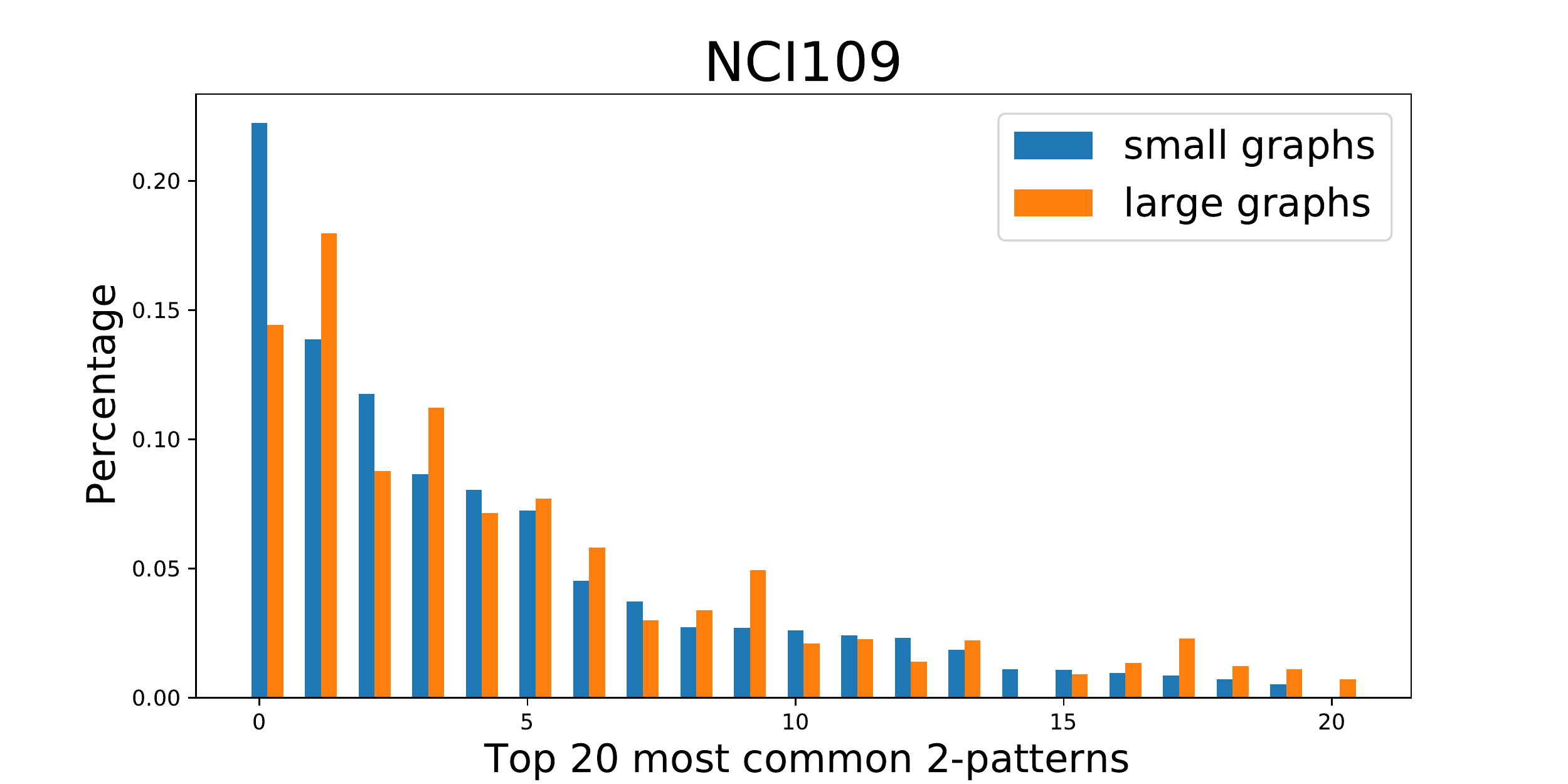}}
    {\includegraphics[width=0.45\textwidth]{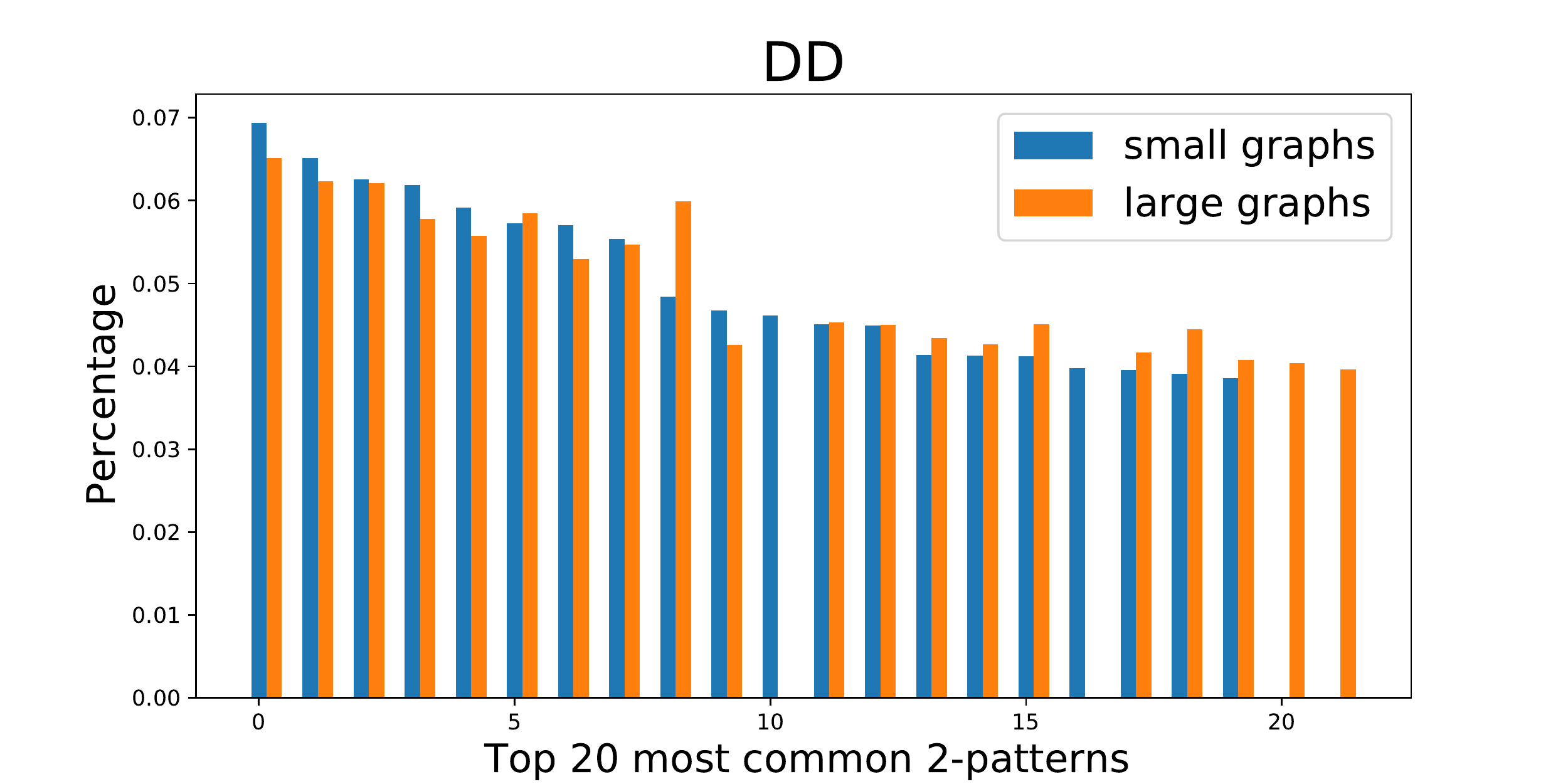}}
    {\includegraphics[width=0.45\textwidth]{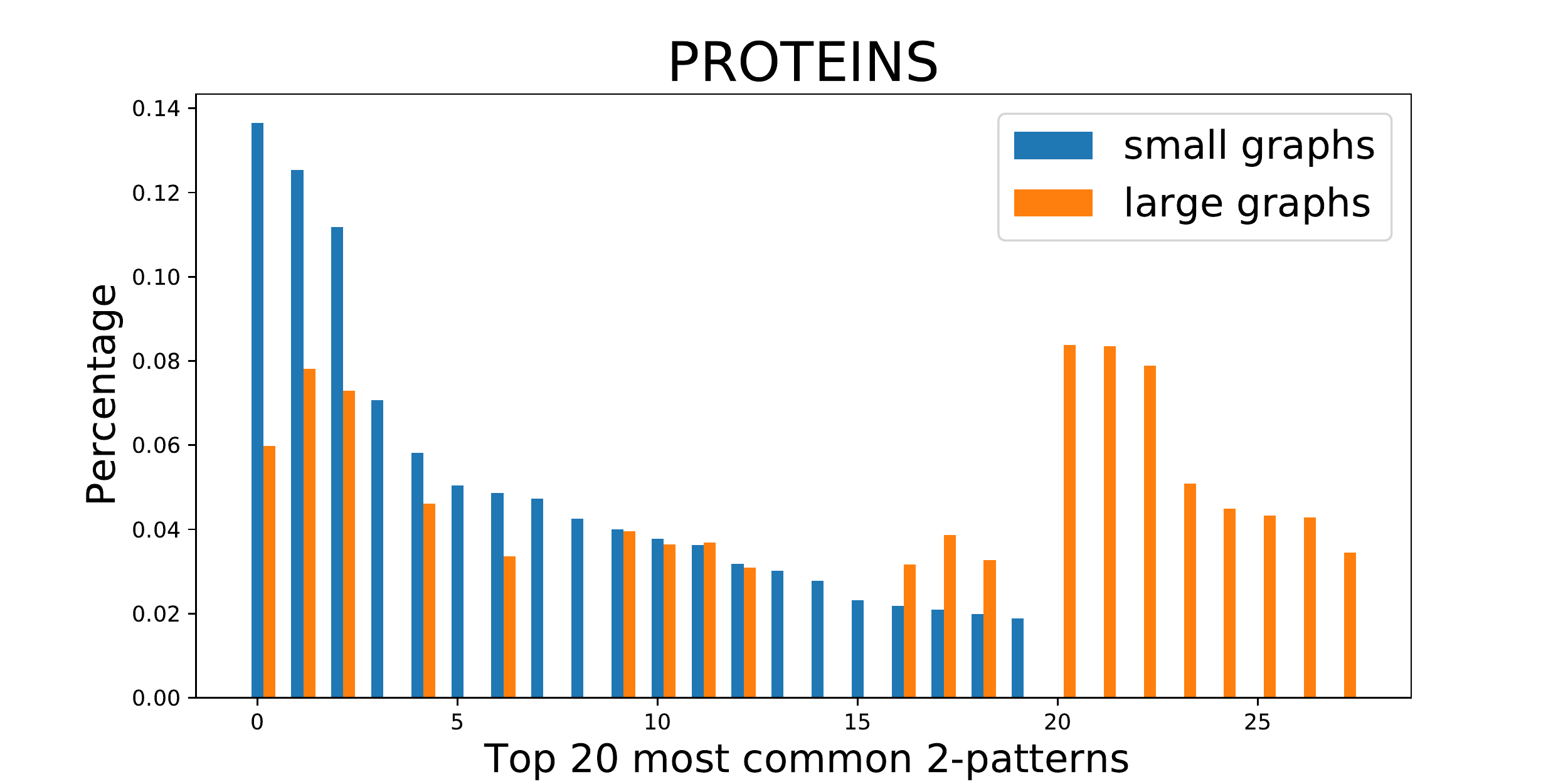}}
    {\includegraphics[width=0.45\textwidth]{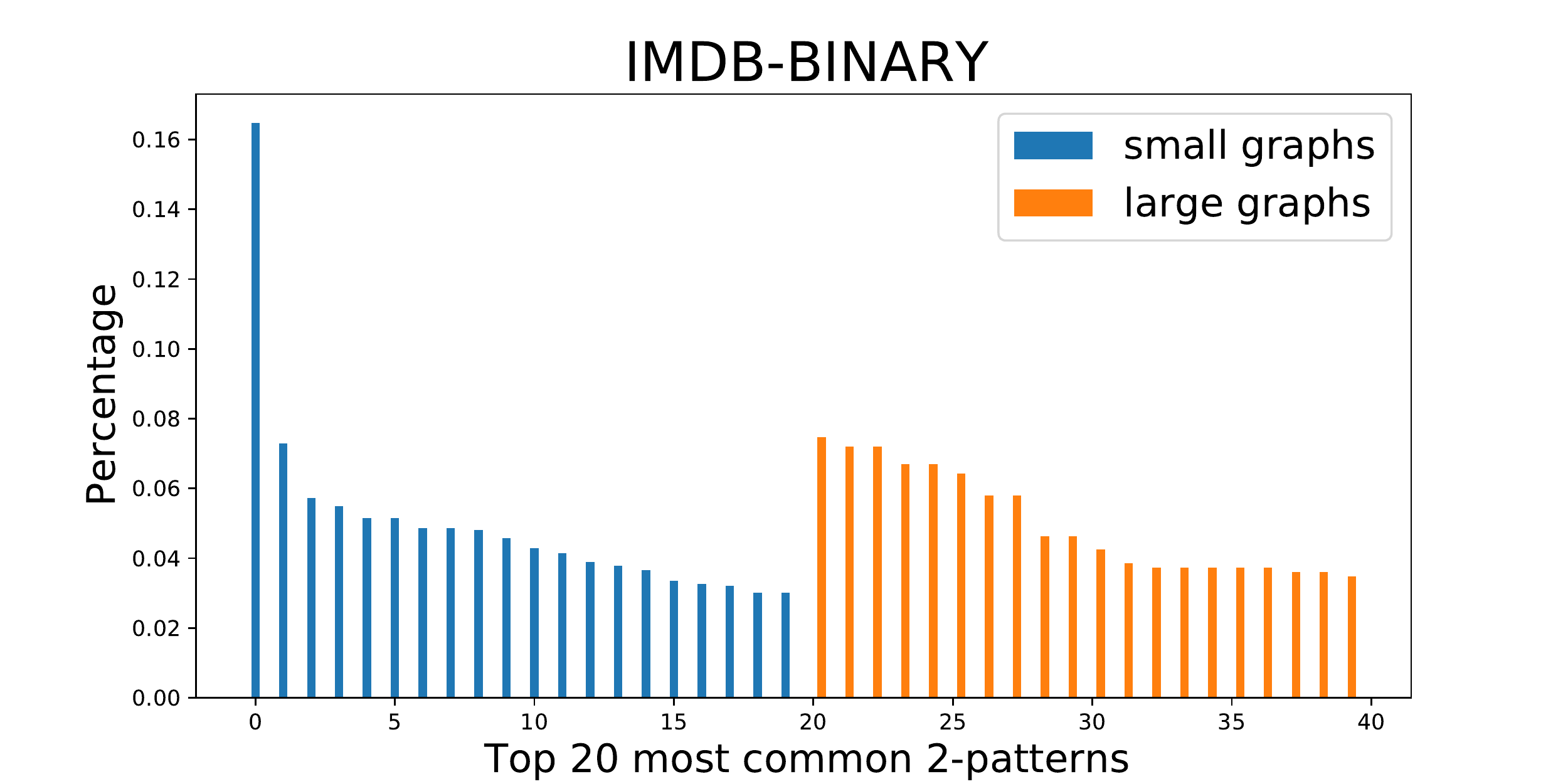}}
    {\includegraphics[width=0.45\textwidth]{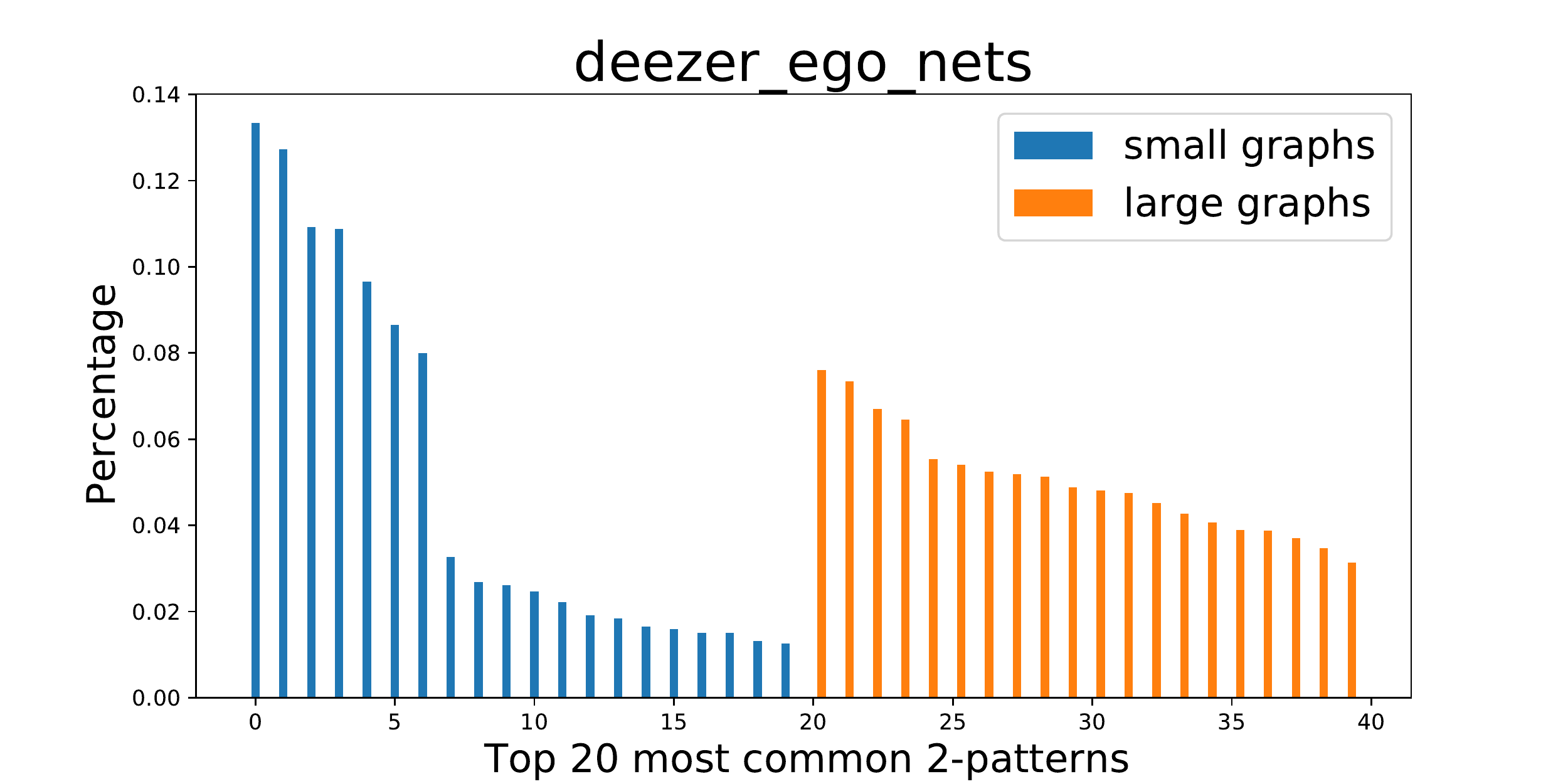}}
    \caption{{Histograms in percentage of $2$-patterns of graphs. We used the 50\% smallest graphs and 10\% largest graphs in each dataset.(same as the split in the experiment from \secref{sec:improve size gen}). The X-axis represent the 20 most common $2$-patterns from each split, and the y-axis their percentage of appearance. The x-axis contain from 20 to 40 bars - given by how much overlap of $2$-patterns there is between small and large graphs.}
    \label{fig:2-pattern hist}}
\end{figure}

\begin{table*}[t]
    \centering
    \scriptsize\begin{sc}
    \begin{tabular}{|l|l|l|l|l|l|}
    \hline
    \textbf{Datasets}    & \textbf{Total-Variation Distance} \\ \hline \hline
     D \& D & 0.15\\ \hline
     NCI1 & 0.16\\ \hline
     NCI109 & 0.16\\ \hline
     Proteins & 0.48\\ \hline
     IMDB-Binary & 0.99\\ \hline
     Deezer ego nets & 1\\ \hline
     Twitch ego nets & 1\\ \hline
    \end{tabular}
    \end{sc}
    \caption{Total variation distance, between the $2$-patterns of the 50\% smallest graphs and 10\% largest graphs for all the real datasets that we tested on. }
    \label{tab:TV distance}
\end{table*}

\end{document}